\documentclass{article}

\PassOptionsToPackage{table}{xcolor}

\usepackage[accepted]{icml2026}

\makeatletter
\renewcommand{\ICML@appearing}{Accepted at the \textit{ICML 2026 Workshop
on Resource-Adaptive Foundation Model Inference (AdaptFM)}, Seoul, South
Korea. Copyright 2026 by the author(s).}
\makeatother

\usepackage[utf8]{inputenc}
\usepackage[T1]{fontenc}
\usepackage{textcomp}
\usepackage{microtype}
\usepackage{nicefrac}

\usepackage{amsmath}
\usepackage{amssymb}
\usepackage{amsfonts}
\usepackage{amsthm}
\usepackage{mathtools}

\usepackage{graphicx}
\usepackage{subcaption}
\usepackage{float}
\usepackage{wrapfig}
\usepackage{booktabs}
\usepackage{tabularx}
\usepackage{longtable}
\usepackage{array}
\usepackage{makecell}
\usepackage{multirow}
\usepackage{colortbl}

\usepackage{enumitem}

\usepackage{soul}
\usepackage{hyperref}
\usepackage[capitalize,noabbrev]{cleveref}

\AtBeginDocument{%
  \hypersetup{%
    colorlinks=false,
    pdfborder={0 0 1},
    citebordercolor={0 1 0},
    linkbordercolor={1 1 1},
    urlbordercolor={1 1 1},
    linkcolor={},
    citecolor={},
    urlcolor={},
  }%
}

\usepackage{balance}

\definecolor{errorred}{RGB}{230,200,255}
\definecolor{grammarorange}{RGB}{255,230,180}
\definecolor{collapsepurple}{RGB}{255,200,200}
\newcommand{\hlerror}[1]{\sethlcolor{errorred}\hl{#1}}
\newcommand{\hlgrammar}[1]{\sethlcolor{grammarorange}\hl{#1}}
\newcommand{\hlcollapse}[1]{\sethlcolor{collapsepurple}\hl{#1}}

\newcolumntype{L}[1]{>{\raggedright\arraybackslash}p{#1}}

\theoremstyle{plain}
\newtheorem{theorem}{Theorem}[section]
\newtheorem{proposition}[theorem]{Proposition}

\theoremstyle{definition}

\theoremstyle{remark}
\newtheorem{remark}[theorem]{Remark}

\definecolor{DarkGreen}{rgb}{0.0, 0.5, 0.0}

\usepackage{tikz}
\usetikzlibrary{positioning,arrows.meta}

\newcommand{\AIR}{\textit{AIR}}

\icmltitlerunning{Activation- and Influence-Aware Ranks (AIR)}

\begin{document}

\twocolumn[
\icmltitle{Activation- and Influence-Aware Ranks (AIR):\\
  Function-Preserving SVD Compression for LLMs}

\icmlsetsymbol{equal}{*}

\begin{icmlauthorlist}
\icmlauthor{Nico Harder}{hhi}
\icmlauthor{Daniel Becking}{hhi}
\icmlauthor{Karsten Mueller}{hhi}
\icmlauthor{Wojciech Samek}{hhi}
\end{icmlauthorlist}

\icmlaffiliation{hhi}{Fraunhofer HHI, Berlin, Germany}

\icmlcorrespondingauthor{Nico Harder}{nico.harder@hhi.fraunhofer.de}

\icmlkeywords{Machine Learning, LLM Compression, SVD, Low-Rank Approximation}

\vskip 0.3in

{\centering
\includegraphics[width=\textwidth]{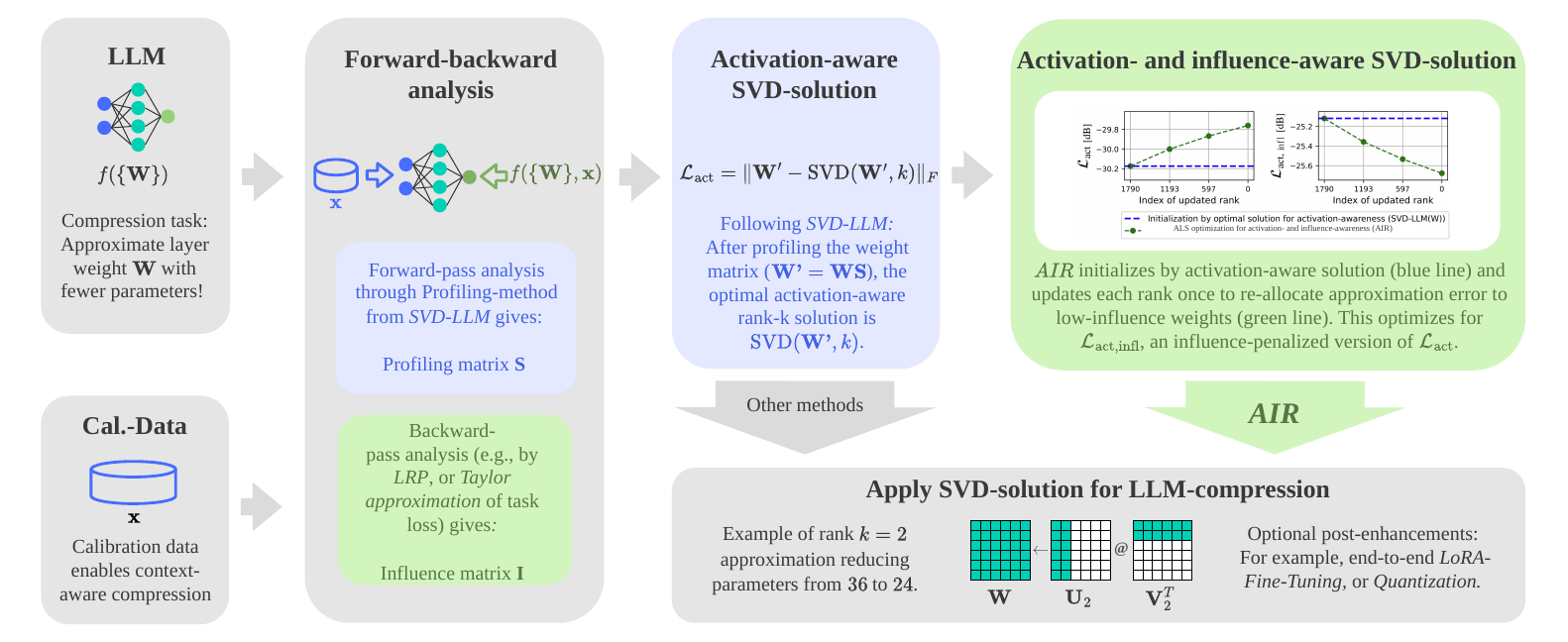}\par}
\vspace{4pt}
{\captionof{figure}{Schematic overview of activation- and influence-aware SVD-based LLM compression through \AIR{}.}\label{fig:rr_update}}

\vskip 0.2in
]

\printAffiliationsAndNotice{}  

\begin{abstract}
We present Activation- and Influence-Aware Ranks (\AIR{}), an SVD-based LLM compression framework that guides each weight matrix's low-rank approximation with a backward-signal influence metric. Starting from the activation-aware optimum of SVD-LLM(W), \AIR{} runs a single closed-form alternating least squares (ALS) sweep that integrates influence element-wise under a monotone-descent guarantee. \AIR{} is layer-local and composes orthogonally with end-to-end methods: alone it exceeds ACIP, and \AIR{}+LoRA outperforms it further. \AIR{} improves perplexity over SVD-LLM(W) by $>$18\% at $\leq$60\% parameter retention, matches its quality with $\sim$90\% less calibration data, and turns parameter savings into FLOP, peak-memory, and per-token latency gains.
\end{abstract}

\section{Introduction}
\label{sec:intro}
 
Large language models (LLMs) gain performance through scale~\cite{kaplan2020scalinglaws}, motivating post-training compression for efficient deployment under varying resource budgets. Singular value decomposition (SVD) is well-suited: transformer weights develop implicit low-rank structure during training~\cite{balzano2025overviewlowrankstructurestraining}, the Eckart--Young theorem~\cite{eckart1936approximation} gives a Frobenius-optimal rank-$k$ approximation, and the rank-hierarchy of singular values provides a built-in importance ordering at the finest granularity possible without calibration data. Vanilla SVD, however, ignores the functional role of each weight and degrades downstream performance severely. Function-preserving SVD methods address this with calibration data: \emph{forward-activation} approaches (ASVD~\cite{yuan2024asvdactivationawaresingularvalue}, SVD-LLM(W)~\cite{svd_llm}) are cheap and layer-local but agnostic to whether the preserved activations influence the final prediction; \emph{backward-signal} approaches expose this functional role yet have either underperformed activation-only baselines (FWSVD~\cite{fwsvd}) or required full end-to-end optimization (ACIP~\cite{genzel2025choosemodelsizecompression}). The gap is not \emph{whether} a backward signal helps, but \emph{how} it is integrated.

\begin{table*}[t!]
\centering
\caption{Comparison of SVD-based LLM compression methods across features and requirements.}
\label{tab:svd_methods_comparison}
\resizebox{\textwidth}{!}{%
\scriptsize
\begin{tabular}{l|cc|ccc|cc}
\toprule
& \multicolumn{2}{c|}{\textsc{Features}} & \multicolumn{3}{c|}{\textsc{Enhancements}} & \multicolumn{2}{c}{\textsc{Costs}} \\
\cmidrule(lr){2-3} \cmidrule(lr){4-6} \cmidrule(lr){7-8}
\textsc{Method} & \makecell{\textsc{Activation-}\\\textsc{Awareness}} & \makecell{\textsc{Influence-}\\\textsc{Awareness}} & \makecell{\textsc{Dyn. Rank}\\\textsc{Allocation}} & \makecell{\textsc{End-to-End}\\\textsc{Optimization}} & \textsc{Quantization} & \makecell{\textsc{Compute-}\\\textsc{Cost}} & \makecell{\textsc{Data-}\\\textsc{Cost}} \\
\midrule
\textcolor{gray}{Vanilla SVD} & \textcolor{gray}{$\times$} & \textcolor{gray}{$\times$} & \textcolor{gray}{Optional} & \textcolor{gray}{Optional} & \textcolor{gray}{Optional} & \textcolor{gray}{Very Low} & \textcolor{gray}{Zero} \\
\midrule
ACIP (\citeauthor{genzel2025choosemodelsizecompression})                          & $\checkmark$ & End-to-end  & Inherent & Inherent & Optional & High   & High   \\
ASVD, ASVD+ (\citeauthor{yuan2024asvdactivationawaresingularvalue})               & $\checkmark$ & $\times$    & Inherent & Optional & Optional & High   & Medium \\
FWSVD (\citeauthor{fwsvd})                                                        & $\times$     & Local proxy & Optional & Optional & Optional & Medium & Medium \\
SVD-LLM(W) (\citeauthor{svd_llm})                                                 & $\checkmark$ & $\times$    & Optional & Optional & Optional & Low    & Medium \\
SVD-LLM (\citeauthor{svd_llm})                                                    & $\checkmark$ & End-to-end  & Optional & Inherent & Optional & Low    & Medium \\
\midrule
\texttt{AIR (Ours)} & $\checkmark$ & Local proxy & Optional & Optional & Optional & Low & Very low \\
\bottomrule
\end{tabular}%
}
\end{table*}

\AIR{} (Activation- and Influence-Aware Ranks) closes this gap (Figure~\ref{fig:rr_update}). A forward-backward analysis yields the profiling matrix $\mathbf{S}$ from activations~\citep{svd_llm} and an element-wise influence matrix $\mathbf{I}$ from a backward signal. \AIR{} jointly optimizes a hybrid activation- and influence-aware objective: starting from $\mathrm{SVD}(\mathbf{W}\mathbf{S}, k)$, per-rank ALS updates redistribute approximation error away from high-influence weights, sidestepping the intractability of element-wise weighted low-rank approximation~\citep{srebro2003weighted} via closed-form per-rank updates with monotone descent. \textbf{Contributions.} \textbf{(i)} On LLaMA-7B, \AIR{} achieves 18\%/33\%/45\% lower WikiText-2 perplexity than SVD-LLM(W) at 60\%/40\%/20\% retention, matches it with $\sim$90\% less calibration data, and generalizes across LLaMA/Mistral/Vicuna families. \textbf{(ii)} A closed-form ALS mechanism for element-wise backward-signal integration that dominates the choice of signal (LRP-$\epsilon$, Weight$\times$Gradient, Fisher all equivalent) and is complementary to end-to-end optimization (\AIR{}+LoRA outperforms ACIP at every rate). \textbf{(iii)} System-level efficiency: 60\% retention turns into 64\% peak memory and 53\% per-token latency on a 40\,GB A100.

\section{Related Work}
\label{sec:related}
\label{subsec:svd_based_compression_for_llms}
\label{subsec:influence_for_compression}

We position \AIR{} against existing SVD-based compression methods (Table~\ref{tab:svd_methods_comparison}). Under \emph{Enhancements}, ``Inherent'' marks methods that bundle the enhancement as part of their core algorithm and ``Optional'' those that treat it as a complementary add-on. ASVD~\cite{yuan2024asvdactivationawaresingularvalue} and SVD-LLM~\cite{svd_llm} rely on \emph{forward activations}: ASVD scales weights by per-channel activation statistics, while SVD-LLM applies a whitening transformation that yields the optimal solution to an activation-aware objective. Methods using \emph{backward signals} include FWSVD~\cite{fwsvd}, which guides compression by row-wise Fisher information, and ACIP~\cite{genzel2025choosemodelsizecompression}, which uses $\ell_1$-regularized gradient descent to learn rank masks and LoRA adapters end-to-end. These differ in how they exploit the signal (Table~\ref{tab:svd_methods_comparison}'s Influence-Awareness column): as a \emph{layer-local proxy} guiding a closed-form solution (FWSVD), or through \emph{end-to-end} optimization (ACIP). The two strategies are not mutually exclusive: SVD-LLM and ACIP inherently bundle a subsequent end-to-end stage, while \AIR{}, FWSVD, and SVD-LLM(W) leave it as an optional, complementary post-step. SVD-LLM(W) (the whitening-only variant) already achieves competitive results and motivates our primary baseline. FWSVD is the closest prior approach to \AIR{}: it combines a backward signal with SVD as a local proxy, yet underperforms activation-aware baselines because (i) its objective uses no forward-pass signal and (ii) it aggregates influence row-wise, since the element-wise case admits no closed-form SVD solution~\cite{srebro2003weighted}. \AIR{}'s ALS updates sidestep this and preserve element-wise granularity. Appendix~\ref{sec:appendix_influence_metrics} organizes the broader family of backward-signal influence metrics (Weight$\times$Gradient~\citep{molchanov2017pruning,shrikumar2017learning}, empirical Fisher~\citep{fwsvd}, OBD/OBS~\citep{lecun1989obd,hassibi1992obs}, LRP~\citep{bach2015pixel,attnlrp,Becking_2022,Yeom_2021,hatefi2024pruningexplainingrevisitedoptimizing}) via the underlying Taylor expansion.

\section{AIR}
\label{sec:method}

This section describes the design of \AIR{} (Figure~\ref{fig:rr_update}); \AIR{} operates on individual layers and we omit layer-indexing. SVD-based compression decomposes a weight matrix $\mathbf{W}\in\mathbb{R}^{m\times n}$ into a rank-$k$ approximation $\mathbf{W}_k=\mathbf{U}_k\boldsymbol{\Sigma}_k\mathbf{V}_k^\top$ ($k\leq\min(m,n)$), Frobenius-optimal by Eckart--Young~\cite{eckart1936approximation}.

\subsection{Forward-backward analysis}

\paragraph{Activations in the forward pass}

Following SVD-LLM~\cite{svd_llm}, we collect hidden states $\mathbf{X}$ over the calibration data and apply a \emph{Profiling} preprocessing to the weight matrix via
\begin{equation}\label{eq:whitening}
\mathbf{W}' = \mathbf{W} \mathbf{S}, \quad \text{where } \mathbf{S} = \text{cholesky}\!\biggl(\sum_{\mathcal{D}_{\text{cal}}}\mathbf{X}\mathbf{X}^\top\biggr).
\end{equation}
Throughout this work, $'$ indicates a variable in the latent space obtained through right-multiplication by the profiling matrix $\mathbf{S}$, which we interpret as the product of the forward-pass analysis.

\paragraph{Influence from the backward pass.}
A backward signal exposes the functional role of each weight by reading its contribution from the model's output side. Since the task loss $\mathcal{L}$ has no tractable closed form in the compressed weights, the standard approach is a Taylor expansion of $\mathcal{L}$ around $\mathbf{W}$:
\begin{equation}\label{eq:taylor_main}
\begin{split}
\mathcal{L}(\hat{\mathbf{W}}_k) &\approx \underbrace{\mathcal{L}(\mathbf{W})}_{\text{0th}} + \underbrace{\nabla_\mathbf{W}\mathcal{L}(\mathbf{W})^\top(\hat{\mathbf{W}}_k-\mathbf{W})}_{\text{1st (slope)}} \\
&\quad+ \underbrace{\tfrac{1}{2}(\hat{\mathbf{W}}_k-\mathbf{W})^\top\mathbf{H}_\mathbf{W}(\hat{\mathbf{W}}_k-\mathbf{W})}_{\text{2nd (curvature)}} + \cdots,
\end{split}
\end{equation}
with the 0th-order term constant in compression. Weight$\times$Gradient~\citep{molchanov2017pruning,shrikumar2017learning} is the first-order term; FWSVD's diagonal empirical Fisher~\citep{fwsvd} approximates the curvature. LRP~\cite{bach2015pixel} arrives at the same first-order quantity from another foundation, interpreting the model's prediction as relevance to redistribute it within each layer under gradient-free propagation rules (Appendix~\ref{sec:appendix_influence_metrics}). \AIR{} accepts any element-wise score $\mathbf{I}\in\mathbb{R}^{m\times n}$ from this family. Empirically, the integration mechanism dominates the choice: principled signals yield identical perplexity through our update rule, while uninformative baselines fall behind (Table~\ref{tab:influence_metric_ablation}). We default to AttnLRP~\cite{attnlrp} with $\epsilon=10^{-6}$~\citep{hatefi2024pruningexplainingrevisitedoptimizing} for its gradient-free rules and noise suppression of near-zero pre-activations: relevance is initialized at $\text{rel}=f(\mathbf{x})$ and propagated to per-layer $\mathbf{R}\in\mathbb{R}^{m\times n}$. \AIR{} sets $\tilde{\mathbf{I}}=\mathbf{R}$, accumulates $\mathbf{I}=\sum_{d\in\mathcal{D}_{\text{cal}}}|\tilde{\mathbf{I}}^{(d)}|$, and normalizes to unit mean per layer.

\subsection{AIR - Problem}\label{subsec:air_problem}
SVD-LLM's \textit{weight-profiling} (Eq.~\ref{eq:whitening}) simplifies the activation-aware objective from $\|\mathbf{WX} - (\mathbf{U}_k \boldsymbol{\Sigma}_k \mathbf{V}_k^\top)\mathbf{X}\|_F^2$ to
\begin{equation}\label{eq:act_loss}
\mathcal{L}_{\text{act}} = \|\mathbf{W}' - \mathbf{U}'_k \boldsymbol{\Sigma}'_k \mathbf{V}'^\top_k\|_F^2,
\end{equation}
minimized optimally by SVD-LLM(W)~\cite{svd_llm} (vanilla-SVD on the \textit{profiled} weight $\mathbf{W}'$, by Eckart--Young~\cite{eckart1936approximation}). \AIR{} expands $\mathcal{L}_{\text{act}}$ with an element-wise influence penalty: each squared residual is scaled by $(1+\delta\cdot i_{ij})$, multiplicatively integrating influence scores $\mathbf{I}$ into the loss. The additive all-ones anchor ensures activation-aware approximation error is not discounted for low-influence weights (revisited empirically in Section~\ref{subsec:als_dynamics}, Figure~\ref{fig:delta}). With $\delta$ controlling the influence impact, \AIR{}'s central hybrid activation- and influence-aware objective reads
\begin{equation}\label{eq:hybrid_loss}
\boxed{\mathcal{L}_{\text{act,infl}} = \bigl\|\sqrt{\mathbf{1}+\delta\cdot\mathbf{I}}\,\odot\,\bigl(\mathbf{W}' - \mathbf{U}'_k\boldsymbol{\Sigma}'_k\mathbf{V}'^\top_k\bigr)\bigr\|_F^2},
\end{equation}
where $\mathbf{I}\in\mathbb{R}^{m\times n}$ is the element-wise influence matrix from the backward-pass analysis (\emph{not} the identity matrix), $\mathbf{1}\in\mathbb{R}^{m\times n}$ is the all-ones matrix, $\odot$ is the Hadamard product, and the square-root is element-wise; at $\delta=0$, $\mathcal{L}_{\text{act,infl}}$ collapses to $\mathcal{L}_{\text{act}}$. The influence-weighted form admits no closed-form SVD solution in general~\citep{srebro2003weighted}; we instead solve $\arg\min_{\mathbf{U}'_k,\boldsymbol{\Sigma}'_k,\mathbf{V}'_k}\mathcal{L}_{\text{act,infl}}$ via the rank-wise closed-form ALS updates derived in Section~\ref{sec:air_solution} (full element-wise penalty derivation in Appendix~\ref{subsec:als_derivation}).

\subsection{AIR - Solution}\label{sec:air_solution}

As initialization, \AIR{} adopts SVD-LLM(W)'s activation-aware optimum $\mathbf{U}'_k, \boldsymbol{\Sigma}'_k, \mathbf{V}'_k \leftarrow \text{SVD}(\mathbf{W}', k)$ (Section~\ref{subsec:air_problem}), then runs an Alternating Least Squares (ALS) sweep over the ranks: a single closed-form update to each rank component $(\sigma'_r, \mathbf{u}'_r, \mathbf{v}'_r)$, iterating $r=k-1$ down to $r=0$. This backward direction protects the leading components (small $r$) from excessive distortion: the first few ALS iterations impose the largest updates and shrink as the residual error converges (Table~\ref{tab:als_ablations}). For each rank $r$, we first update $\mathbf{v}'_r$ given $\mathbf{U}'_k$, $\boldsymbol{\Sigma}'_k$, and the residual $\mathbf{E}_r = \mathbf{W}' - \mathbf{W}'_k + \sigma'_r \mathbf{u}'_r \mathbf{v}'^\top_r$ (i.e., $\mathbf{W}'$ minus the current rank-$k$ approximation, with rank $r$'s contribution added back), through
\begin{equation}\label{eq:update_v}
\mathbf{v}'_r = \left( \frac{\mathbf{u}'^\top_r \bigl((\mathbf{1} + \delta \cdot \mathbf{I}) \odot \mathbf{E}_r\bigr)}{\sigma'_r \cdot (\mathbf{u}'^2_r)^\top (\mathbf{1} + \delta \cdot \mathbf{I})} \right)^\top,
\end{equation}
where the division is element-wise and $\mathbf{u}'^2_r$ denotes element-wise squaring. Then, given $\boldsymbol{\Sigma}'_k$ and the updated $\mathbf{V}'_k$, we obtain
\begin{equation}\label{eq:update_u}
\tilde{\mathbf{u}}'_r = \frac{\bigl((\mathbf{1} + \delta \cdot \mathbf{I}) \odot \mathbf{E}_r\bigr) \mathbf{v}'_r}{(\mathbf{1} + \delta \cdot \mathbf{I}) (\mathbf{v}'^2_r)},
\end{equation}
again element-wise, and extract $\sigma'_r = \|\tilde{\mathbf{u}}'_r\|_2$, $\mathbf{u}'_r = \tilde{\mathbf{u}}'_r / \sigma'_r$. Canceling $\mathbf{v}'_r$ between numerator and denominator is not possible: these matrix-vector products mix different components of $\mathbf{v}'_r$ across summations. 

From this initialization (already the global optimum of $\mathcal{L}_{\text{act}}$ by SVD-LLM(W)), every ALS step provably refines the perturbed objective:
\begin{proposition}[Monotone descent of the ALS sweep]\label{prop:monotone_main}
Starting from $\mathrm{SVD}(\mathbf{W}', k)$ and applying Eqs.~\ref{eq:update_v}--\ref{eq:update_u} for $r=k-1,\dots,0$, $\mathcal{L}_{\text{act,infl}}$ is non-increasing at every step (proof in Appendix~\ref{subsec:als_derivation}).
\end{proposition}

\paragraph{Applying the \AIR{} solution to LLM compression.} We project the decomposition back into the native space and absorb the singular values: $\mathbf{U}_k \leftarrow \mathbf{U}'_k \sqrt{\boldsymbol{\Sigma}'_k}$, $\mathbf{V}^\top_k \leftarrow \sqrt{\boldsymbol{\Sigma}'_k} (\mathbf{V}'^{\top}_k \mathbf{S}^{-1})$. Compression is realized through $\mathbf{W} \leftarrow \mathbf{U}_k \mathbf{V}_k^\top$; at inference, $\mathbf{y} = \mathbf{U}_k(\mathbf{V}_k^\top\mathbf{x})$ reduces MACs from $mn$ to $k(m+n)$ (full FLOP analysis in Appendix~\ref{subsec:implementation_efficiency}; pseudocode in Algorithm~\ref{alg:RR}).

\section{Experiments}\label{sec:experiments}

This section evaluates the performance of \AIR{} against and in combination with other compression methods. Some SVD-based methods inherently require enhancements such as dynamic rank allocation or gradient-based fine-tuning (Sec.~\ref{sec:related}). To facilitate fair comparisons, each experiment groups methods with similar requirements: Table~\ref{tab:perplexity_comparison} considers methods without enhancements; Table~\ref{tab:perplexity_comparison_lora} filters for methods using post-compression end-to-end fine-tuning through LoRA. Each method receives identical calibration data and is evaluated with the same protocol. \textbf{Default experimental configuration.} Unless otherwise specified, our experiments employ: model \texttt{jeffwan/llama-7b-hf} at 16-bit half-precision, evaluation on \textsc{WikiText-2}, calibration with 256 \textsc{WikiText-2} samples of 2048 tokens each, AttnLRP with the Epsilon rule at $\epsilon=10^{-6}$, influence weighting strength $\delta=2.0$, parameter rate $60\%$, token generation at batch size $64$ with sequence length $512$. Best results in bold, relative gains over second-best in green.
 
\begin{table}[t!]
\centering
\caption{SVD-based compression of LLaMA 7B without enhancements: WikiText-2 / C4 perplexity ($\downarrow$) and average reasoning accuracy ($\uparrow$, mean over OpenBookQA, ARC-E, WinoGrande, HellaSwag, PIQA, MathQA). Full table including JS-divergence and per-task reasoning in Appendix~\ref{subsec:full_perplexity_comparison} (Table~\ref{tab:perplexity_comparison_full}).}
\label{tab:perplexity_comparison}
\resizebox{\linewidth}{!}{%
\scriptsize
\setlength{\tabcolsep}{3.5pt}
\begin{tabular}{l|l|cc|c}
\toprule
    \textsc{Param.} & \textsc{Method} & \multicolumn{2}{c|}{\textsc{Perplexity} $\downarrow$} & \textsc{Reas.} $\uparrow$ \\
    \cmidrule(lr){3-4}
    \textsc{Rate} & & WikiText2 & C4 & Avg.\ \\
\midrule
    \textcolor{gray}{100\%} & \textcolor{gray}{Base Model} & \textcolor{gray}{5.68} & \textcolor{gray}{7.34} & \textcolor{gray}{57\%} \\
\midrule
    & Vanilla SVD & 19438 & 16115 & 32.8\% \\
    & FWSVD & 22026 & 32048 & 31.3\% \\
    80\% & ASVD & 116 & 105 & 39.5\% \\
    & SVD-LLM(W) & 7.87 & 16.65 & 48.8\% \\
    \cmidrule{2-5}
    & \texttt{AIR} & \textbf{7.51} (\textcolor{DarkGreen}{$\downarrow$4.6\%}) & \textbf{14.24} (\textcolor{DarkGreen}{$\downarrow$14.5\%}) & \textbf{49.9\%} (\textcolor{DarkGreen}{$\uparrow$2.3\%}) \\
\midrule
    & Vanilla SVD & 52839 & 46630 & 32.1\% \\
    & FWSVD & 81838 & 111860 & 31.1\% \\
    60\% & ASVD & 4915 & 8103 & 31.3\% \\
    & SVD-LLM(W) & 13.81 & 56.33 & 40.0\% \\
    \cmidrule{2-5}
    & \texttt{AIR} & \textbf{11.27} (\textcolor{DarkGreen}{$\downarrow$18.4\%}) & \textbf{35.81} (\textcolor{DarkGreen}{$\downarrow$36.4\%}) & \textbf{41.6\%} (\textcolor{DarkGreen}{$\uparrow$4.0\%}) \\
\midrule
    & Vanilla SVD & 105082 & 105082 & 31.8\% \\
    & FWSVD & 134928 & 126754 & 31.5\% \\
    40\% & ASVD & 143631 & 134928 & 31.4\% \\
    & SVD-LLM(W) & 63.83 & 345 & 33.3\% \\
    \cmidrule{2-5}
    & \texttt{AIR} & \textbf{42.52} (\textcolor{DarkGreen}{$\downarrow$33.4\%}) & \textbf{277} (\textcolor{DarkGreen}{$\downarrow$19.7\%}) & \textbf{33.6\%} (\textcolor{DarkGreen}{$\uparrow$0.9\%}) \\
\midrule
    & Vanilla SVD & 196320 & 208981 & \textbf{32.4\%} (\textcolor{DarkGreen}{$\uparrow$2.2\%}) \\
    & FWSVD & 268337 & 208981 & 31.2\% \\
    20\% & ASVD & 24959 & 23447 & 31.2\% \\
    & SVD-LLM(W) & 854 & 8626 & 31.3\% \\
    \cmidrule{2-5}
    & \texttt{AIR} & \textbf{472} (\textcolor{DarkGreen}{$\downarrow$44.7\%}) & \textbf{2550} (\textcolor{DarkGreen}{$\downarrow$70.4\%}) & 31.7\% \\
\bottomrule
\end{tabular}
}
\end{table}

\subsection{Performance-efficiency trade-off}

\paragraph{Methods without enhancements.} Our primary efficiency metric is the \emph{parameter count rate} $k(m+n)/(m\cdot n)$ for a rank-$k$ approximation of $\mathbf{W}\in\mathbb{R}^{m \times n}$, applied uniformly across layers (a reliable proxy for FLOP, memory, and runtime gains; Table~\ref{tab:compression_metrics_correlation}). Table~\ref{tab:perplexity_comparison} compares \AIR{} and SVD-based baselines~\cite{svd_llm,yuan2024asvdactivationawaresingularvalue} on WikiText-2~\cite{merity2016pointer}/C4~\cite{raffel2020exploring} perplexity and average common-sense reasoning accuracy~\cite{clark2018think, mihaylov2018can, zellers2019hellaswag, bisk2020piqa, sakaguchi2021winogrande, amini2019mathqa} (full per-task and JS-divergence~\cite{jsd_khanal2024evaluatingimpactcompressiontechniques} numbers, including 80\%/20\% retention rows, in Appendix Table~\ref{tab:perplexity_comparison_full}). \AIR{}'s closed-form ALS adds only $\sim$12 minutes of A100 compute on LLaMA-7B and remains complementary to LoRA (Table~\ref{tab:perplexity_comparison_lora}). \AIR{} consistently outperforms SVD-LLM(W), with the gap widening at higher compression, evidence that simultaneous activation- and influence-awareness preserves parameters that purely activation-aware methods cannot distinguish. FWSVD performs worse than vanilla SVD, reflecting the two issues \AIR{} addresses: weak (row-wise) signal integration combined with no activation-awareness. \textbf{Generalization across models.} Table~\ref{tab:perplexity_comparison_other_models} extends evaluation to Mistral 7B~\cite{jiang2023mistral7b}, Vicuna-7B~\cite{vicuna2023}, TinyLlama (1.1B)~\cite{zhang2024tinyllamaopensourcesmalllanguage}, LLaMA 2-7B~\cite{touvron2023llama2}, LLaMA 3-8B~\cite{llama3modelcard}, and LLaMA 30B~\cite{touvron2023llama}: \AIR{} achieves the lowest perplexity at 60\% rate on every model. \textbf{Structured-pruning comparison.} Table~\ref{tab:perplexity_comparison_structural_pruning} compares \AIR{} against SliceGPT~\cite{ashkboos2024slicegpt} and BlockPruner~\cite{zhong2025blockprunerfinegrainedpruninglarge} at matched memory footprints; \AIR{}'s advantage grows under tighter memory budgets.

\begin{table}[t!]
\centering
\caption{Perplexity on different models, evaluated at 60\% parameter rate.}
\label{tab:perplexity_comparison_other_models}
\resizebox{\linewidth}{!}{%
\scriptsize
\renewcommand{\arraystretch}{1.05}
\setlength{\tabcolsep}{3pt}
\begin{tabular}{l|c|cc|c}
\toprule
& \textcolor{gray}{Base Model} & ASVD & SVD-LLM(W) & \texttt{AIR} \\
\midrule
Mistral-7B       & \textcolor{gray}{5.24} & 18261 & 87.2  & \textbf{60} (\textcolor{DarkGreen}{$\downarrow$31.2\%}) \\
Vicuna-7B        & \textcolor{gray}{6.78} & 3827  & 18.58 & \textbf{16.65} (\textcolor{DarkGreen}{$\downarrow$10.36\%}) \\
TinyLLaMA (1.1B) & \textcolor{gray}{7.99} & 9182  & 39.33 & \textbf{29.68} (\textcolor{DarkGreen}{$\downarrow$24.53\%}) \\
LLaMA 2-7B       & \textcolor{gray}{5.49} & 3596  & 16.65 & \textbf{14.47} (\textcolor{DarkGreen}{$\downarrow$13.09\%}) \\
LLaMA 3-8B       & \textcolor{gray}{6.13} & 19438 & 230   & \textbf{67.95} (\textcolor{DarkGreen}{$\downarrow$70.46\%}) \\
LLaMA 30B        & \textcolor{gray}{4.11} & 16.91 & 7.74  & \textbf{7.16}  (\textcolor{DarkGreen}{$\downarrow$7.49\%}) \\
\bottomrule
\end{tabular}}
\end{table}

\begin{table}[t!]
\centering
\caption{\AIR{} vs.\ structured pruning, LLaMA 7B (13.5GB), perplexity ($\downarrow$).}
\label{tab:perplexity_comparison_structural_pruning}
\resizebox{\linewidth}{!}{%
\scriptsize
\renewcommand{\arraystretch}{1.05}
\setlength{\tabcolsep}{3pt}
\begin{tabular}{l|cccc}
\toprule
\textsc{Memory} & 10GB & 9GB & 8GB & 7GB \\
\midrule
BlockPruner & 9.88* & 12.21* & 18.94* & 21.68* \\
SliceGPT    & 9.49  & 12.18  & 17.18  & 26.12  \\
\midrule
\texttt{AIR}
& \textbf{8.24}  (\textcolor{DarkGreen}{$\downarrow$13.2\%})
& \textbf{9.64}  (\textcolor{DarkGreen}{$\downarrow$20.9\%})
& \textbf{12.37} (\textcolor{DarkGreen}{$\downarrow$28.0\%})
& \textbf{18.29} (\textcolor{DarkGreen}{$\downarrow$30.0\%}) \\
\bottomrule
\end{tabular}}\\[2pt]
{\scriptsize * Value taken from~\cite{svd_llm}.}
\end{table}

\begin{table}[t!]
\centering
\captionof{table}{LoRA fine-tuning enhancement, perplexity ($\downarrow$). Further end-to-end ablations (full retraining on the calibration set, with and without LoRA) are reported in Appendix~\ref{subsec:retraining_discussion}.}
\label{tab:perplexity_comparison_lora}
\resizebox{\linewidth}{!}{%
\scriptsize
\renewcommand{\arraystretch}{1.05}
\setlength{\tabcolsep}{3pt}
\begin{tabular}{l|cccc}
\toprule
\textsc{Param. rate} & 80\% & 60\% & 40\% & 20\% \\
\midrule
ACIP (LoRA inherent) & 8.64 & 12.57 & 25    & 3378 \\
SVD-LLM(W) + LoRA       & 7.31 & 9.41  & 16.82 & 114  \\
\midrule
\texttt{AIR + LoRA}
& \textbf{7.24}  (\textcolor{DarkGreen}{$\downarrow$1\%})
& \textbf{9.09}  (\textcolor{DarkGreen}{$\downarrow$3\%})
& \textbf{14.49} (\textcolor{DarkGreen}{$\downarrow$14\%})
& \textbf{46.86} (\textcolor{DarkGreen}{$\downarrow$59\%}) \\
\bottomrule
\end{tabular}}

\vspace{1.0em}

\centering
\captionof{table}{\AIR{} at 60\% parameter rate stacked with RTN/GPTQ quantization across 8/4/2-bit widths (LLaMA 7B, base 13.5\,GB); file-size rate ($\downarrow$) and WikiText-2 perplexity ($\downarrow$).}
\label{tab:perplexity_comparison_quantization}
\resizebox{\linewidth}{!}{%
\scriptsize
\renewcommand{\arraystretch}{1.05}
\setlength{\tabcolsep}{3pt}
\begin{tabular}{l|cc cc cc}
\toprule
\textsc{Bit Width}
& \multicolumn{2}{c}{8-bit}
& \multicolumn{2}{c}{4-bit}
& \multicolumn{2}{c}{2-bit} \\
\cmidrule(lr){2-3} \cmidrule(lr){4-5} \cmidrule(lr){6-7}
& Mem. & PPL & Mem. & PPL & Mem. & PPL \\
\midrule
\texttt{AIR(60\%) + RTN}  & 35.49\% & 11.44 & 20.63\% & 13.17 & 13.19\% & 81838 \\
\texttt{AIR(60\%) + GPTQ} & 35.47\% & 11.27 & 20.61\% & 11.99 & 13.17\% & 17154 \\
\bottomrule
\end{tabular}}

\vspace{1.0em}

\centering
{\scriptsize
\renewcommand{\arraystretch}{1.05}
\begin{tabular*}{\linewidth}{@{\extracolsep{\fill}}l|cc|ccc@{}}
\toprule
& \multicolumn{2}{c|}{\textsc{Pseudo}} & \multicolumn{3}{c}{\textsc{Backward, Principled}} \\
\cmidrule(lr){2-3} \cmidrule(lr){4-6}
\textsc{Signal} ($\mathbf{I}$) & \textsc{Ones} & \textsc{W} & \textsc{W$\times$Grad} & \textsc{Fisher} & \textsc{LRP-$\epsilon$} \\
\midrule
\texttt{AIR(60\%)} & 13.80 & 14.02 & 11.27 & 11.27 & \textbf{11.27} \\
\bottomrule
\end{tabular*}}
\captionof{table}{Perplexity at different signals (magnitudes only) populating influence matrix $\mathbf{I}$. W denotes weights.}
\label{tab:influence_metric_ablation}

\vspace{0.8em}

\centering
{\scriptsize
\renewcommand{\arraystretch}{1.05}
\begin{tabular*}{\linewidth}{@{\extracolsep{\fill}}l|cc|ccc@{}}
\toprule
& \multicolumn{2}{c|}{\textsc{Direction}} & \multicolumn{3}{c}{\textsc{Sweep count}} \\
\cmidrule(lr){2-3} \cmidrule(lr){4-6}
& \textsc{Backw.} & \textsc{Forw.} & 1 & 2 & 10 \\
\midrule
\texttt{AIR(60\%)} & \textbf{11.27} & 11.81 & \textbf{11.27} & 11.44 & 11.44 \\
\bottomrule
\end{tabular*}}
\captionof{table}{ALS ablations: perplexity across sweep directions and counts.}
\label{tab:als_ablations}
\end{table}

\begin{figure}[t!]
\centering
\includegraphics[width=\linewidth,height=0.18\textheight,keepaspectratio,trim=0 0 0 6, clip]{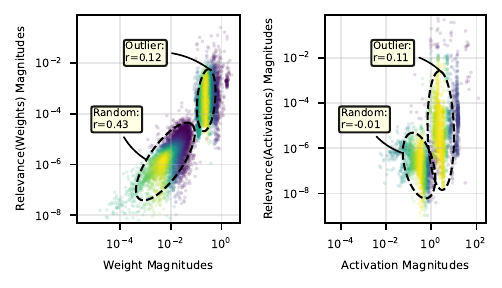}
\captionof{figure}{Weight (left) and activation (right) magnitudes vs.\ attributed relevance across all layers of LLaMA 7B; \textit{Random} samples and \textit{Outliers}.}
\label{fig:scatter_correlation}

\vspace{1.0em}

\centering
\includegraphics[width=\linewidth,height=0.20\textheight,keepaspectratio]{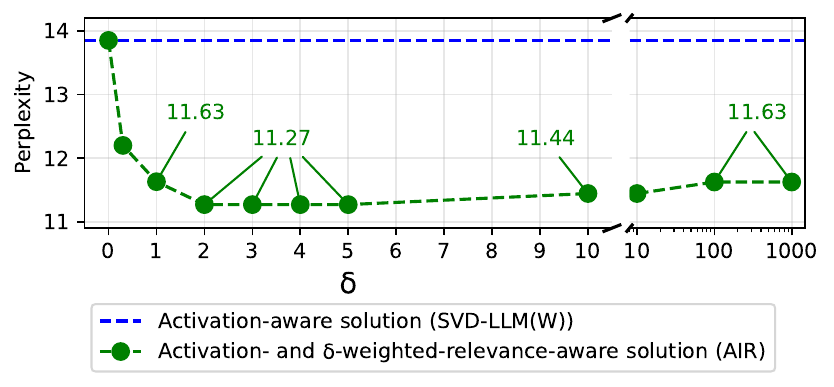}
\captionof{figure}{Perplexity vs.\ influence-weighting strength~$\delta$, yielding the optimal working point.}
\label{fig:delta}
\end{figure}

\paragraph{Methods with enhancements.} \textbf{End-to-end fine-tuning through LoRA.} Table~\ref{tab:perplexity_comparison_lora} evaluates LoRA fine-tuning~\cite{Hu2021} after compression. \AIR{}\,+\,LoRA and SVD-LLM(W)\,+\,LoRA use Alpaca~\cite{dubois2024alpacafarmsimulationframeworkmethods} with 50k samples (sequential U/V optimization following~\cite{svd_llm}); ACIP~\cite{genzel2025choosemodelsizecompression} bundles LoRA inherently with 1k fine-tuning steps~\cite{merantix2024acip_repo}. Although \AIR{} and end-to-end fine-tuning both exploit a backward signal, they remain complementary rather than redundant: \AIR{} alone is already competitive with end-to-end ACIP, and \AIR{}+LoRA outperforms every baseline at every rate (gap up to 59\% at 20\% retention). Appendix~\ref{subsec:retraining_discussion} extends this to full retraining of the compressed model on the calibration set; \AIR{}, LoRA, and full retraining each contribute independent gains. \textbf{Quantization.} Table~\ref{tab:perplexity_comparison_quantization} combines \AIR{} with RTN and GPTQ~\cite{Frantar2023GPTQ}. At 60\% rate with 4-bit GPTQ, \AIR{} retains 11.99 perplexity (vs.\ 11.27 unquantized) at 20.6\% memory, indicating the low-rank structures \AIR{} learns are amenable to lower-precision representation.

\subsection{ALS dynamics and influence-weighting}\label{subsec:als_dynamics}
The integration mechanism, not the signal choice, drives \AIR{}'s gains: ablating the influence source (Table~\ref{tab:influence_metric_ablation}), all three principled backward signals (LRP-$\epsilon$, Weight$\times$Gradient, diagonal Fisher) yield identical perplexity, while uninformative baselines (ones, weight magnitude) fall behind. The inset of Figure~\ref{fig:rr_update} visualizes the ALS dynamics in a single layer (1228 ranks at 60\% parameter rate): $\mathcal{L}_{\text{act,infl}}$ decreases monotonically from the activation-only initialization to the \AIR{} solution while $\mathcal{L}_{\text{act}}$ correspondingly grows, evidence that ALS re-allocates reconstruction error from high- to low-influence positions. The sweep completes in 1.5\,s per layer on an A100 ($\sim$12\,min model-wide, including profiling and influence accumulation); a single backward sweep ($r=k-1\!\to\!0$) is optimal because the larger early updates then land on subordinate rather than dominant ranks, and additional sweeps continue to lower $\mathcal{L}_{\text{act,infl}}$ but slightly degrade perplexity (Table~\ref{tab:als_ablations}), consistent with mild overfitting to the calibration-set proxy. \textbf{Influence-weighting strength.} Figure~\ref{fig:delta} shows perplexity vs.\ $\delta$ in Eq.~\ref{eq:hybrid_loss}: $\delta=0$ recovers SVD-LLM(W), the optimum sits at $\delta=2.0$ (our default), and perplexity rises mildly at large $\delta$ where the objective effectively degenerates to a purely influence-weighted reconstruction $(\mathbf{1}+\delta\mathbf{I})\to\delta\mathbf{I}$, empirically confirming the role of the all-ones anchor. \textbf{Influence vs.\ native-space magnitudes.} Figure~\ref{fig:scatter_correlation} shows correlations between weight (left) and activation (right) magnitudes and their attributed relevance for two subsets per layer of LLaMA 7B: (1) \textit{Random}, $1$~ppm of total parameters, and (2) \textit{Outliers}, the $1$~ppm highest-magnitude weights/activations. Random-sample weights show only modest correlation; the relationship essentially vanishes for outlier weights and for activations in either subset. Considering the high impact of outliers in LLMs~\cite{xiao2023smoothquant}, this disconnect indicates that neither weight nor activation magnitude reliably proxies influence, creating headroom for compression that preserves high-relevance, low-magnitude parameters that magnitude-based methods would discard (spatial view in Appendix~\ref{sec:appendix_spatial_patterns}).

\begin{table}[t!]
\centering
{\scriptsize
\renewcommand{\arraystretch}{1.05}
\begin{tabular*}{\linewidth}{@{\extracolsep{\fill}}c!{\vrule}l!{\vrule}cccc@{}}
\toprule
\textsc{Type} & \textsc{Metric} & \multicolumn{4}{c}{\textsc{compr.\,/\,base}} \\
\midrule
\multirow{3}{*}{\textsc{Static}}
    & \textsc{Param.}   & 80\% & 60\% & 40\% & 20\% \\
    & \textsc{FLOPs}    & 80\% & 61\% & 41\% & 21\% \\
    & \textsc{Mem.}     & 81\% & 62\% & 42\% & 23\% \\
\midrule
\multirow{2}{*}{\textsc{Runtime}}
    & \textsc{Peak GPU} & 75\% & 64\% & 52\% & 42\% \\
    & \textsc{Latency}  & 53\% & 53\% & 52\% & 53\% \\
\bottomrule
\end{tabular*}}
\caption{Parameter rate is a reliable proxy for holistic SVD-based efficiency.}
\label{tab:compression_metrics_correlation}
\end{table}

\begin{figure}[t!]
\centering
\includegraphics[width=\linewidth,height=0.21\textheight,keepaspectratio]{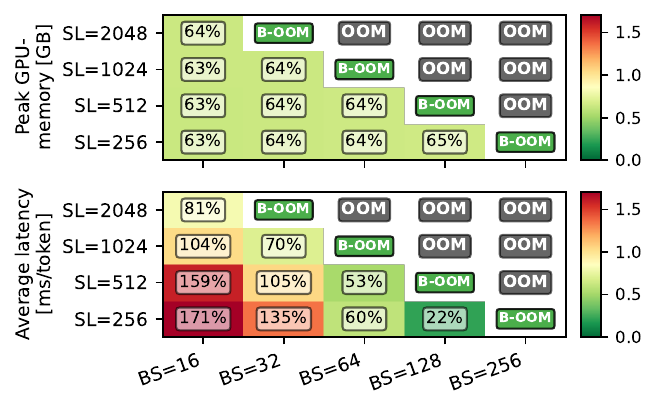}
\caption{Runtime efficiency: Peak GPU-memory and per-token latency rates across batch sizes and sequence lengths; LLaMA 7B at 60\% parameter rate. ``OOM'' marks cells exceeding the 40\,GB A100 budget; ``B-OOM'' marks cells where only the base model OOMs while \AIR{}(60\%) still fits. Absolute values in Appendix Fig.~\ref{fig:peak_mem_throughput_rate_full}.}
\label{fig:peak_mem_throughput_rate}
\end{figure}

\subsection{Efficiency}

\textbf{Token-generation.} Realizing the efficiency gains promised by SVD-based compression requires a forward-pass implementation that matches the parameter reduction at every level. Static gains do not arise automatically: memory reductions track the parameter count only when truncated ranks are physically pruned rather than zero-masked, and FLOP reductions require sequential propagation $\mathbf{y}=\mathbf{U}_k(\mathbf{V}_k^\top\mathbf{x})$, avoiding the preliminary full-weight reconstruction $\mathbf{W}=\mathbf{U}_k\mathbf{V}_k^\top$. Runtime gains during autoregressive decoding demand additional care: KV-cache low-rank caching, pre-allocated buffers, and RoPE pre-application. \AIR{} integrates all of these into its inference path; at 60\% parameter rate on LLaMA-7B, this yields a near-uniform peak GPU-memory ratio of ${\sim}64\%$ across the batch-size$\times$sequence-length grid (Figure~\ref{fig:peak_mem_throughput_rate}). Per-token latency shows two regimes: under high bandwidth pressure (large BS and/or long SL, near the OOM frontier) \AIR{} reduces latency the most, down to $22\%$ at BS=128, SL=256; at small batch sizes and short sequences the two-matmul overhead can exceed the base model's latency (up to $171\%$ at BS=16, SL=256). At our default operating point (BS=64, SL=512), \AIR{} reaches $53\%$ of base latency, matching the near-constant latency rate at this fixed reference (Table~\ref{tab:compression_metrics_correlation}). The additional payoff of more aggressive compression is the matched peak-GPU reduction, which shifts the OOM frontier outward and unlocks the near-OOM regime where the largest latency gains arise; e.g., \AIR{}(60\%) fits BS=64, SL=1024, where the base model OOMs on a 40\,GB A100. Prior SVD-based methods often achieve the static numbers but overclaim runtime gains; Appendix~\ref{subsec:implementation_efficiency} details our optimizations and audits these claims, including a corrected accounting that turns a reported $2.65\times$ SVD-LLM speedup into a $27\%$ slowdown once early-EOS termination is accounted for. \textbf{Data efficiency.} \AIR{} matches SVD-LLM(W)'s 256-sample quality with $\sim$90\% less calibration data, and the gap persists as more samples are added (Figure~\ref{fig:26_samples}); if gains came purely from better activation estimation, more samples would close the gap.

\begin{figure}[!t]
\centering
\includegraphics[width=\linewidth,height=0.18\textheight,keepaspectratio]{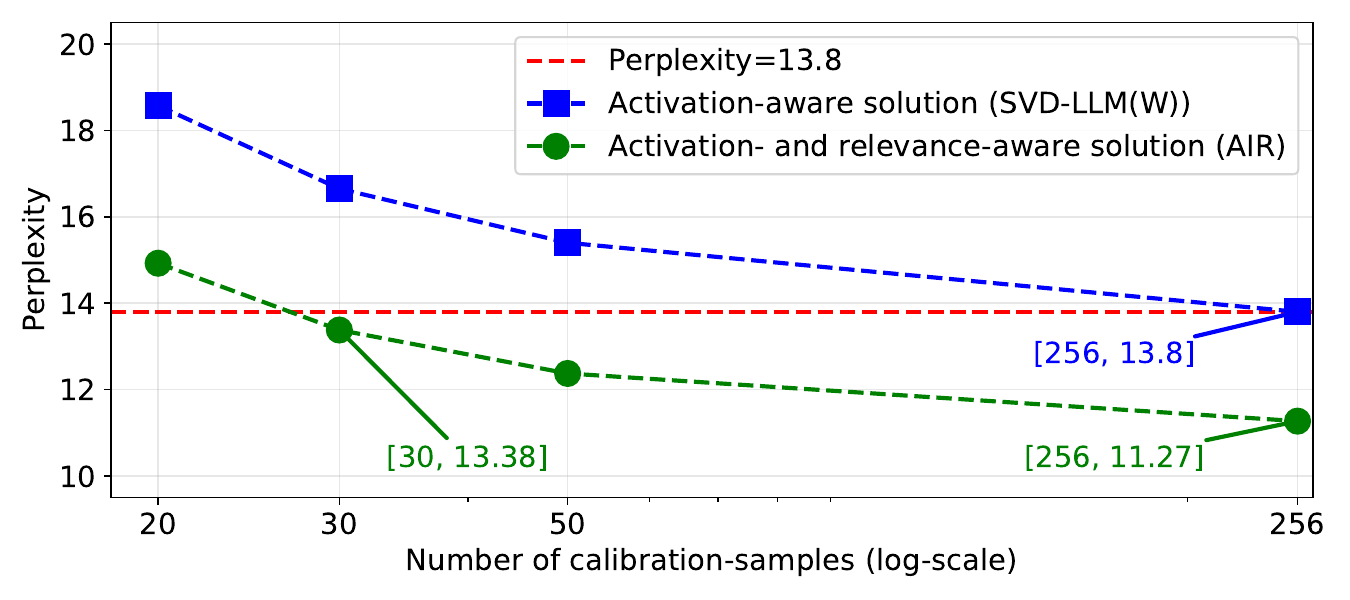}
\caption{Data-efficiency: WikiText-2 perplexity vs.\ number of calibration samples for \AIR{} and SVD-LLM(W).}
\label{fig:26_samples}
\end{figure}

\section{Conclusion}
\label{sec:conclusion}

\AIR{} closes the gap between activation-aware SVD compression and end-to-end retraining in a single closed-form, layer-local sweep. On LLaMA-7B it lowers perplexity over SVD-LLM(W) at all parameter rates (widening under aggressive compression), matches it with $\sim90\%$ less calibration data, and converts parameter savings into runtime gains ($\sim64\%$ peak memory, $53\%$ latency at 60\% retention on a 40\,GB A100); gains carry over to Mistral, Vicuna, LLaMA-2/3/30B and to structured-pruning baselines. Two takeaways. (1) \emph{Layer-local and end-to-end optimization are complementary}: LoRA, full retraining (Appendix~\ref{subsec:retraining_discussion}), and GPTQ quantization each compose without erasing \AIR{}'s advantage. (2) \emph{The integration mechanism dominates the choice of signal}: principled backward signals yield equivalent perplexity through our update rule. \textbf{Limitations.} Scaling to 70B+ and encoder-decoder/MoE architectures remains open; \AIR{}'s uniform per-layer rate leaves dynamic rank allocation, token-level adaptation, and budget-conditional inference as natural extensions. \textbf{Code:} \url{https://github.com/NicodeHarder/air}.

\clearpage

\balance
\bibliographystyle{icml2026}
\bibliography{99_references}

\onecolumn
\appendix

\section{Appendix}

\paragraph{Organization.} Appendix~\ref{sec:appendix_influence_metrics} develops the functional-awareness framework in full, including the Taylor-expansion hierarchy that organizes gradient-, curvature-, and Fisher-based influence metrics, and the mathematical equivalence between LRP-0 and Weight$\times$Gradient that supports the signal-agnostic claim of Sec.~\ref{sec:method}. Appendix~\ref{subsec:als_derivation} derives the closed-form ALS updates and proves Propositions~\ref{prop:init}--\ref{prop:monotone}. Appendix~\ref{subsec:als_dynamics_ext} provides the extended ablations and relevance-property analyses that the condensed Sec.~\ref{subsec:als_dynamics} summarizes (influence-signal ablation, ALS sweep direction/count, scatter and spatial heatmaps). Appendix~\ref{subsec:implementation_efficiency} presents the systems-level analysis of SVD-based inference, including an audit of throughput and KV-cache claims in prior work and the forward-pass optimizations that realize end-to-end gains. Appendix~\ref{subsec:retraining_discussion} reports a further end-to-end ablation that stacks \AIR{} with full retraining on the calibration set, in addition to and in combination with LoRA fine-tuning, providing the most demanding test of the additivity reading discussed in Sec.~\ref{sec:experiments}. Appendix~\ref{subsec:sentence_completion} provides the qualitative compression-failure analysis referenced in the main text, and Appendix~\ref{sec:appendix_spatial_patterns} visualizes the spatial structure of influence within transformer weight matrices.

\subsection{Functional Awareness in Compression}
\label{sec:appendix_influence_metrics}

Recall our general SVD-based compression through $\hat{\mathbf{W}}_k = \mathbf{U}_k\boldsymbol{\Sigma}_k\mathbf{V}_k^\top$, the rank-$k$ approximation of $\mathbf{W}$. Vanilla SVD minimizes weight reconstruction error $\|\mathbf{W} - \hat{\mathbf{W}}_k\|_F$, treating all weight elements as equally important regardless of their functional role in the network.
Achieving \emph{functional awareness}, ensuring that compression preserves the model's input--output behavior, requires incorporating calibration data $\mathcal{D}_\text{cal}$ to distinguish functionally critical from redundant parameters.
Existing approaches can be categorized by whether they derive this awareness from the \emph{forward pass} (activation-based, local) or the \emph{backward pass} (loss- or output-based, global).

Forward-pass methods operate layer-locally: they are computationally cheap and can yield provably optimal solutions to their respective local objectives, but remain unaware of whether the preserved activations actually contribute to the final prediction.
Backward-pass methods address this limitation by considering the global prediction loss or model output, but since directly minimizing the global objective amounts to costly retraining, they rely on \emph{influence metrics} as tractable proxies.
\AIR{} combines both: activation-awareness from the forward pass with influence-awareness from the backward pass.

Throughout, $\mathbf{W} \in \mathbb{R}^{m \times n}$ denotes a layer weight matrix, $x_i$ denotes the hidden state (activation) of input neuron $i$, pre-activations are $z_{ij} = x_i w_{ij}$ with $z_j = \sum_i z_{ij}$, and $f(\mathbf{x})$ denotes the model output.
We omit layer indexing for notational simplicity.

\subsubsection*{Forward-Pass Methods: Local Activation-Awareness}

For brevity in this overview, we use the equivalent scalar Frobenius form for $\mathcal{L}_{\text{act}}$; see Section~\ref{sec:method} (Eq.~\ref{eq:act_loss}) for the element-wise matrix form used throughout the \AIR{} derivation, with the two related by $\sum_{ij}[\mathcal{L}_{\text{act}}]_{ij} = \|\cdot\|_F^2$.

Forward-pass methods aim to preserve the integrity of the local layer output $\mathbf{Y} = \mathbf{W}\mathbf{X}$ under compression, minimizing
$\mathcal{L}_\text{act} = \|\mathbf{W}\mathbf{X} - \hat{\mathbf{W}}_k \mathbf{X}\|_F^2$.

\paragraph{ASVD~\citep{yuan2024asvdactivationawaresingularvalue}.}
ASVD introduces a diagonal scaling matrix $\mathbf{S}_0 = \mathrm{diag}(s_1, \ldots, s_n)$, where each $s_j$ reflects the significance of input channel $j$, typically set as $s_j = (\overline{|x_j|})^\alpha$ with $\overline{|x_j|}$ denoting the mean absolute activation of channel $j$ and $\alpha$ as a tunable exponent.
ASVD decomposes the scaled weight $\mathbf{W}\mathbf{S}_0$ via SVD, then absorbs $\mathbf{S}_0^{-1}$ into the reconstruction: $\mathbf{Y} \approx (\mathbf{U}_k \boldsymbol{\Sigma}_k \mathbf{V}_k^\top) \mathbf{S}_0^{-1} \mathbf{X}$.
This heuristically reweights the SVD to account for activation magnitudes across channels, but does not provably minimize $\mathcal{L}_\text{act}$; the scaling is approximate and the choice of $\alpha$ requires tuning.

\paragraph{SVD-LLM~\citep{svd_llm}.}
SVD-LLM applies a whitening transformation $\mathbf{W}' = \mathbf{W}\mathbf{S}$ with $\mathbf{S} = \mathrm{chol}\!\left(\sum_{\mathcal{D}_\text{cal}}\mathbf{X}\mathbf{X}^\top\right)$, absorbing the full second-order activation structure into the weight.
Since SVD of $\mathbf{W}'$ yields the optimal rank-$k$ approximation in Frobenius norm (Eckart--Young), and $\mathcal{L}_\text{act} = \|\mathbf{W}' - \hat{\mathbf{W}}'_k\|_F^2$ after profiling, this provides the \emph{provably optimal} activation-aware solution, a key advantage over ASVD's heuristic scaling.

Both methods are efficient (requiring only forward passes over $\mathcal{D}_\text{cal}$) but share a key limitation: they optimize for local layer-output fidelity without awareness of whether the preserved activations actually contribute to the model's final prediction.

\subsubsection*{Backward-Pass Methods: Global Prediction-Awareness
via Influence Metrics}

To incorporate global prediction-awareness, backward-pass methods consider the task loss of the compressed model.
Denoting the full set of model parameters by
$\boldsymbol{\theta}$, the task loss reads:
\begin{equation}
    \mathcal{L}(\boldsymbol{\theta}) =
    \frac{1}{|\mathcal{D}_\text{cal}|}
    \sum_{(\mathbf{x},y) \in \mathcal{D}_\text{cal}}
    \ell\!\left(y,\; f(\boldsymbol{\theta}, \mathbf{x})\right),
    \label{eq:task_loss}
\end{equation}
where $\ell$ denotes the per-sample loss (e.g., cross-entropy).
The ideal compression minimizes
$\mathcal{L}(\hat{\boldsymbol{\theta}}_k)$, the task loss evaluated
at the compressed parameters $\hat{\boldsymbol{\theta}}_k$.
Directly minimizing this objective amounts to training, which
can be prohibitively expensive for LLMs.
Instead, backward-pass methods propagate information from the model
output to yield \emph{influence scores} for weight components.
The compression task then decomposes into two steps:
(1)~defining an effective influence score $i_{ij}$ for each weight
$w_{ij}$, which quantifies the perturbation effect of modifying it
during compression; and
(2)~integrating this influence to guide compression, e.g., by
finding compressed weights $\hat{w}_{ij}$ that minimize the
influence-weighted surrogate
$\sum_{i,j} i_{ij}(w_{ij} - \hat{w}_{ij})^2$.

\paragraph{Defining Influence via Taylor Expansion.}
The influence metrics discussed below can be understood as
successive approximations to the task loss at the compressed weights
$\mathcal{L}(\hat{\boldsymbol{\theta}}_k)$: the derivatives are
computed at the original parameters $\boldsymbol{\theta}$, while
the expansion is evaluated at $\hat{\boldsymbol{\theta}}_k$, so
one simultaneously differentiates the task loss \emph{and}
approximates the compression objective.
The Taylor expansion is used here not to directly
optimize $\hat{\boldsymbol{\theta}}_k$ (which would require
inverting the Hessian or iterative descent), but to \emph{probe the
loss landscape} around $\boldsymbol{\theta}$: the resulting
per-element scores characterize how sensitive $\mathcal{L}$ is to
each weight, and these scores then guide a tractable surrogate
optimization (see ``Integrating Influence'' below).
The expansion is simplified by considering one layer at a time
(ignoring cross-layer dependencies), so that compression of
$\mathbf{W}$ to $\hat{\mathbf{W}}_k$ is the only perturbation.
The Taylor expansion of the loss around $\mathbf{W}$, evaluated at $\hat{\mathbf{W}}_k$, then gives:
\begin{equation}
\begin{split}
    \mathcal{L}(\hat{\mathbf{W}}_k)
    &\approx \underbrace{\mathcal{L}(\mathbf{W})}_{\text{0th order}}
    + \underbrace{\nabla_\mathbf{W}\mathcal{L}(\mathbf{W})^\top
      (\hat{\mathbf{W}}_k - \mathbf{W})}_{\text{1st order}} \\
    &\quad + \underbrace{\tfrac{1}{2}\,
      (\hat{\mathbf{W}}_k - \mathbf{W})^\top
      \mathbf{H}_\mathbf{W}\,
      (\hat{\mathbf{W}}_k - \mathbf{W})}_{\text{2nd order}}
    + \;\cdots\;,
\end{split}
\label{eq:taylor_general}
\end{equation}
where $\mathbf{H}_\mathbf{W} =
\nabla^2_\mathbf{W}\mathcal{L}(\mathbf{W})
\in \mathbb{R}^{p \times p}$ is the Hessian restricted to the
$p = mn$ parameters of this layer.
Since the 0th-order term $\mathcal{L}(\mathbf{W})$ is a constant
independent of the compression, it is irrelevant for minimization,
and influence metrics are derived from the remaining terms.

\paragraph{First-Order Influence (gradXweight).}
Retaining only the first-order term and restricting to a
single-element perturbation ($\hat{w}_{ij} = 0$, all others
unchanged, so $\hat{w}_{ij} - w_{ij} = -w_{ij}$) gives
the estimated loss change
$-w_{ij} \cdot \partial\mathcal{L}/\partial w_{ij}$.
This first-order term captures the \emph{slope} of the loss
landscape: it measures how steeply $\mathcal{L}$ changes when
$w_{ij}$ is perturbed. The signed influence reads:
\begin{equation}
    i_{ij}^{\text{gradXweight}} = -w_{ij} \cdot
    \frac{\partial \mathcal{L}}{\partial w_{ij}}
    = -z_{ij} \cdot
    \frac{\partial \mathcal{L}}{\partial z_j},
    \label{eq:imp_taylor}
\end{equation}
since $\partial\mathcal{L}/\partial w_{ij} =
(\partial\mathcal{L}/\partial z_j) \cdot x_i$.
This is identical to
``Weight\,$\times$\,Gradient''~\citep{molchanov2017pruning,shrikumar2017learning};
the two names reflect different motivations but yield the same formula.
The first-order approach ignores curvature and all higher-order
terms, yet this approximation has proved effective
empirically and receives further legitimacy through its mathematical
equivalence to LRP-0-based influence (see below), an
independent framework that arrives at the same quantity.
First-order influence also remains the only computationally feasible
option among the Taylor-derived metrics; although the
gradient is computed at the \emph{original} weights $\mathbf{W}$, it
yields scores that guide compression toward $\hat{\mathbf{W}}_k$.
The reformulation in terms of pre-activations $z_{ij}$ facilitates
comparison with LRP below.

\paragraph{Second-Order Influence (Hessian and Approximations).}
Retaining only the second-order term captures the \emph{curvature}
of the loss landscape: it measures how the slope itself changes
around $w_{ij}$, so that even a weight with small gradient can
be important if the loss curves sharply in its vicinity.
Again restricting to a single element:
\begin{equation}
    i_{ij}^{\text{Hessian}} = \tfrac{1}{2}\,w_{ij}^2 \cdot
    \frac{\partial^2\mathcal{L}}{\partial w_{ij}^2},
    \label{eq:imp_hessian}
\end{equation}
where $\partial^2\mathcal{L}/\partial w_{ij}^2$ is a diagonal
element of $\mathbf{H}_\mathbf{W}$.
Note that $(-w_{ij})^2 = w_{ij}^2$, so this term is always
non-negative (assuming positive semi-definite curvature) and
requires no absolute value. Computing $\mathbf{H}_\mathbf{W}$ requires $O(p^2)$ storage, infeasible for LLMs, motivating two lines of
approximation:

\textbf{(a) Direct structural approximation.}
Optimal Brain Damage~\citep{lecun1989obd} retains only the diagonal
of $\mathbf{H}_\mathbf{W}$; Optimal Brain
Surgeon~\citep{hassibi1992obs} additionally considers off-diagonal
elements via inverse-Hessian row computations.

\textbf{(b) Fisher information as Hessian proxy.}
FWSVD~\citep{fwsvd} replaces $\mathbf{H}_\mathbf{W}$ with the
empirical Fisher
$\hat{\mathbf{F}} = \frac{1}{|\mathcal{D}|}\sum_d
(\nabla_w\mathcal{L}_d)(\nabla_w\mathcal{L}_d)^\top$,
reducing the second derivative to a sum of squared first-order
gradients. This substitution captures second-order curvature only
under restrictive conditions on the data and parameters that are
essentially never met in practice~\citep{martens2020new,kunstner2019limitations}.
FWSVD further simplifies $\mathbf{F}$ to a diagonal, yielding
per-element influence
$i_{ij}^{\text{Fisher}} = \frac{1}{|\mathcal{D}|}\sum_d
(\partial\mathcal{L}_d / \partial w_{ij})^2$.
GFWSVD~\citep{gfwsvd} recovers some correlation
structure via Kronecker factorization.

\paragraph{Note on terminology: ``Hessian'' in the LLM compression literature.}
The term ``Hessian'' appears in two distinct senses in this literature, and we adopt only one of them throughout this section. SparseGPT, Wanda, OBC, SlimGPT, and SVD-LLM rely on the \emph{layer-wise reconstruction Hessian} $\mathbf{X}^\top\mathbf{X}$, i.e.\ the Hessian of the local objective $\|\mathbf{W}\mathbf{X} - \hat{\mathbf{W}}\mathbf{X}\|_F^2$; in our taxonomy this is an activation-driven local quantity, and ASVD and SVD-LLM accordingly describe the same object as ``activation-aware'' rather than as a Hessian (forward-pass column of Table~\ref{tab:influence_summary}). The Taylor-derived ``Hessian'' of this section is the \emph{loss Hessian} $\partial^2\mathcal{L}/\partial\mathbf{w}^2$ targeted by OBD, a distinct curvature object. To our knowledge, the diagonal of the loss Hessian has not been used as an element-wise saliency for LLM compression: diagonal Fisher provides a cheaper proxy, and exact diagonal estimation via Hutchinson / Hessian-vector-product estimators (cf.\ AdaHessian, Sophia) is substantially more expensive at LLM scale.

\paragraph{Unified view and approximation chain.}
First-order Taylor, Hessian, and Fisher all characterize the
\emph{same} global loss $\mathcal{L}$
(Eq.~\ref{eq:task_loss}) at increasing derivative
orders (slope, then curvature) with the Fisher being the
cheapest second-order proxy (approximating curvature
via squared gradients).
Influence-based compression in general is built on a chain of
approximations: (i)~considering only one layer at a time, ignoring
cross-layer dependencies; (ii)~retaining only a single term of an
infinite Taylor series; and (iii)~for second-order methods,
restricting the Hessian to its diagonal or a low-rank surrogate.
These are simplifications, yet in practice they provide
effective guidance for compression.
The first-order Taylor influence also finds independent
validation through Layer-wise Relevance Propagation: as we
show below, it is mathematically equivalent to LRP-0.
Despite this equivalence, LRP is derived from a distinct
foundation, not based on loss gradients but on relevance initialized
at the model output and propagated backward under layer-wise
conservation.
This offers a gradient-free perspective on influence,
which has also been shown empirically effective for
pruning~\citep{Yeom_2021,hatefi2024pruningexplainingrevisitedoptimizing}.

\paragraph{LRP-Based Relevance.}
LRP~\citep{bach2015pixel} assigns influence through a different
mechanism: rather than expanding $\mathcal{L}$, it
decomposes a scalar model output $f_c(\mathbf{x})$, the logit of
the predicted token, via relevance conservation.
For the $\epsilon$-rule at weight $w_{ij}$, with pre-activations
$z_{ij} = x_i w_{ij}$ and $z_j = \sum_i z_{ij}$:
\begin{equation}
    r_{i \leftarrow j}^{\text{LRP-}\epsilon} =
    \frac{z_{ij}}{z_j + \epsilon \cdot \mathrm{sign}(z_j)}
    \cdot R_j,
    \label{eq:lrpe_weight}
\end{equation}
where $R_j$ is the relevance received from the layer above.
This shares the forward factor $z_{ij}$ with first-order Taylor but
differs in the backward signal: Taylor propagates
$\partial\mathcal{L}/\partial z_j$ (loss gradient), while LRP
propagates $R_j / z_j$ (relevance ratio).

For deep ReLU networks with the basic LRP-0 rule ($\epsilon = 0$),
piecewise linearity causes the redistribution ratios $z_{ij}/z_j$ to
coincide with the local linear
coefficients~\citep{shrikumar2017learning,montavon2019layer}.
Chaining from the output to any weight $w_{ij}$ yields:
\begin{equation}
    r_{i \leftarrow j}^{\text{LRP-0}}
    = \frac{z_{ij}}{z_j} \cdot R_j
    = z_{ij} \cdot \frac{\partial f_c}{\partial z_j}
    = w_{ij} \cdot \frac{\partial f_c}{\partial w_{ij}},
    \label{eq:lrp0_taylor_equiv}
\end{equation}
where the magnitude is the same as from Weight\,$\times$\,Gradient when applied to some logit $f_c$ (instead of $\mathcal{L}$). This identity holds for any scalar function propagated from the output; initializing LRP from $\mathcal{L}$ instead of $f_c$ would recover exactly the first-order Taylor term (Eq.~\ref{eq:imp_taylor}), up to sign. The practical difference, $f_c$ vs.\ $\mathcal{L}$, is thus one of initialization choice: since cross-entropy mixes all logits through the softmax, the two signals differ in general, but align closely when the predicted token dominates the output distribution.

The $\epsilon$-rule introduces a separate deviation: the stabilizer
prevents the cancellation of $z_j$ in
Eq.~\eqref{eq:lrp0_taylor_equiv}, producing influence estimates
distinct from any gradient-based quantity.
For small $\epsilon$ (e.g., $\epsilon = 10^{-6}$ as used in this
work), this deviation is negligible whenever $|z_j| \gg \epsilon$,
and the two criteria yield near-identical influence landscapes.

\paragraph{Integrating Influence into Compression.}
Regardless of how influence is defined, all methods discussed above
yield per-element scores that must be integrated into the compression
objective.
Following established
practice~\citep{Yeom_2021,hatefi2024pruningexplainingrevisitedoptimizing},
influence is represented by its \emph{magnitude only}: the sign of
the Taylor term or LRP relevance indicates the \emph{direction} of
effect (beneficial vs.\ detrimental perturbation), but for guiding
compression we care only about the \emph{size} of the functional
impact, since reconstruction errors can have either sign.
\AIR{} does not use influence scores merely to rank or
select weights but instead \emph{scales} the approximation error
element-wise, mathematically representing the full first-order Taylor
product (not just one of its factors).
Concretely, influence magnitudes are accumulated over calibration
data and collected in the influence matrix
$\mathbf{I}_\mathbf{W} \in \mathbb{R}^{m \times n}$, then
normalized to unit mean per layer.
This normalization ensures that a weight with average influence
receives no modification (multiplicative factor of one), while
functionally critical weights are penalized more heavily for
reconstruction error.
The normalized influence enters the hybrid objective (Eq.~\ref{eq:hybrid_loss}) by reweighting each squared residual element-wise with $(1 + \delta \cdot i_{ij})$, where $\delta$ controls the influence weighting strength.
A detailed derivation of the resulting ALS update rules is provided
in Subsection~\ref{subsec:als_derivation}.

\subsubsection{Summary of functional awareness in SVD-based compression}
Table~\ref{tab:influence_summary} provides a comparative overview.
Forward-pass methods offer provably optimal local solutions (SVD-LLM) but lack prediction-awareness.
Loss-based backward-pass methods (Taylor, Hessian, Fisher) form a hierarchy of increasingly sophisticated, and increasingly approximated, characterizations of the global task loss.
LRP-based relevance is derived from a distinct foundation: it decomposes the model's prediction rather than the loss, avoids the Fisher approximation cascade, and provides element-wise influence with built-in noise suppression via the $\epsilon$-rule (which becomes more pronounced at larger $\epsilon$).
\AIR{} exploits this complementarity by combining activation-awareness (forward-pass, $\mathbf{X}$) with task-awareness (backward-pass influence signal, $\mathbf{I}_\mathbf{W}$, defaulting to LRP-$\epsilon$) in its hybrid objective (Eq.~\ref{eq:hybrid_loss}). The condensed ablation in Sec.~\ref{sec:experiments} (Table~\ref{tab:influence_metric_ablation}) shows that the \AIR{} mechanism is robust to the specific backward signal used to fill $\mathbf{I}_\mathbf{W}$.

\begin{table*}[t!]
\centering
\caption{Overview of functional-awareness strategies for SVD-based compression. ``Scope'' indicates local (layer) vs.\ global (model) awareness. For backward-pass influence metrics, ``Order'' refers to the derivative order of $\mathcal{L}$. The diagonal of the loss Hessian (OBD row) is computable via Hutchinson / Hessian-vector-product estimators (cf.\ AdaHessian, Sophia) but is substantially more expensive than diagonal Fisher and has not, to our knowledge, been used as an element-wise saliency for LLM compression. $^{\dagger}$LRP-$\epsilon$ inherits 1st-order Taylor character via the LRP-0 equivalence (Eq.~\ref{eq:lrp0_taylor_equiv}), with the $\epsilon$-stabilizer adding noise suppression.}
\label{tab:influence_summary}
\footnotesize
\setlength{\tabcolsep}{4pt}
\renewcommand{\arraystretch}{1.15}
\begin{tabularx}{\textwidth}{@{}l l c c c X c@{}}
    \toprule
    \textbf{Method} & \textbf{Type} & \textbf{Scope} & \textbf{Order} & \textbf{Data} & \textbf{What it captures} & \textbf{\makecell[c]{PPL at 60\%\\(as $\mathbf{I}$ in \AIR{})}} \\
    \midrule
    Vanilla SVD & -- & Local & -- & No & Weight reconstruction error & -- \\
    \midrule
    ASVD & Forward & Local & -- & Yes & Heuristic activation scaling & -- \\
    SVD-LLM & Forward & Local & -- & Yes & Optimal activation-aware reconstruction & -- \\
    \midrule
    Ones & -- & Global & 0th & No & Dummy attribution & 13.80 \\
    Weight magnitude & -- & Global & 0th & No & Parameter size & 14.02 \\
    \midrule
    Gradient & Backward & Global & 1st & Yes & Loss sensitivity & -- \\
    \makecell[l]{Taylor 1st order\\(Weight$\times$Gradient)} & Backward & Global & 1st & Yes & Estimated loss change from removal & 11.27 \\
    Diagonal Fisher & Backward & Global & Pseudo-2nd & Yes & Diagonal-Hessian curvature proxy & 11.27 \\
    Hessian (full) & Backward & Global & 2nd & Yes & Loss curvature & infeasible \\
    Hessian (OBD, diagonal) & Backward & Global & 2nd & Yes & Loss curvature, diagonal only & not evaluated \\
    \midrule
    LRP-$\epsilon$ (\AIR{} default) & Backward & Global & 1st$^{\dagger}$ & Yes & Functional contribution to $f(\mathbf{x})$ & \textbf{11.27} \\
    \bottomrule
\end{tabularx}
\end{table*}

\subsection{ALS Optimization}\label{subsec:als_derivation}

We provide the complete derivation of the closed-form updates for our influence-weighted ALS optimization. Algorithm~\ref{alg:RR} summarizes the resulting procedure as pseudocode.

\begin{algorithm}[ht]
\caption{AIR}
\label{alg:RR}
\begin{algorithmic}[1]
\REQUIRE Weight matrix $\mathbf{W}$, profiling matrix $\mathbf{S}$, influence matrix $\mathbf{I}$, influence weighting strength $\delta$, target rank $k$
\ENSURE Activation- and influence-aware compression $\mathbf{U}_k\mathbf{V}_k^\top$
\STATE Apply profiling: $\mathbf{W}' \leftarrow \mathbf{W}\mathbf{S}$
\STATE Initialize via SVD: $\mathbf{U}'_k, \boldsymbol{\Sigma}'_k, \mathbf{V}'_k \leftarrow \text{SVD}(\mathbf{W}', k)$
\STATE Compute $\mathbf{W}'_k \leftarrow \mathbf{U}'_k \boldsymbol{\Sigma}'_k \mathbf{V}'^\top_k$
\FOR{$r = k-1$ \textbf{down to} $0$}
    \STATE Compute residual: $\mathbf{E}_{r} \leftarrow \mathbf{W}' - \mathbf{W}'_k + \sigma'_r \mathbf{u}'_r \mathbf{v}'^\top_r$
    \STATE Update $\mathbf{v}'_r$ via Eq.~\eqref{eq:update_v}
    \STATE Update $\mathbf{u}'_r, \sigma'_r$ via Eq.~\eqref{eq:update_u}
    \STATE Update $\mathbf{W}'_k$ incrementally
\ENDFOR
\STATE Reverse profiling: $\mathbf{V}'^\top_k \leftarrow \mathbf{V}'^\top_k \mathbf{S}^{-1}$
\STATE Absorb singular values: $\mathbf{U}_k \leftarrow \mathbf{U}'_k \sqrt{\boldsymbol{\Sigma}'_k}$, $\mathbf{V}^\top_k \leftarrow \sqrt{\boldsymbol{\Sigma}'_k} \mathbf{V}'^\top_k$
\end{algorithmic}
\end{algorithm}

\paragraph{Problem formulation}

Expanding Eq.~\eqref{eq:hybrid_loss} element-wise, we aim to minimize:
\begin{equation}
\mathcal{L}(\mathbf{U}'_k, \boldsymbol{\Sigma}'_k, \mathbf{V}'_k) = \sum_{i,j} (1 + \delta \cdot i_{ij}) \bigl(\mathbf{W}'_{ij} - [\mathbf{U}'_k \boldsymbol{\Sigma}'_k \mathbf{V}'^\top_k]_{ij}\bigr)^2
\end{equation}
where $\delta \geq 0$ controls the influence weighting strength: when $\delta = 0$, we recover standard (unweighted) low-rank approximation; when $\delta > 0$, elements with higher influence $i_{ij}$ receive greater weight in the reconstruction loss.

The reconstruction can be written as a sum of rank-1 components:
\begin{equation}
\mathbf{W}'_k = \mathbf{U}'_k \boldsymbol{\Sigma}'_k \mathbf{V}'^\top_k = \sum_{r=0}^{k-1} \sigma'_r \mathbf{u}'_r \mathbf{v}'^\top_r
\end{equation}
where $\sigma'_r$ denotes the $r$-th singular value, and $\mathbf{u}'_r$, $\mathbf{v}'_r$ are the corresponding left and right singular vectors.

\paragraph{Coordinate descent strategy}

ALS optimizes one rank-$r$ component $(\sigma'_r, \mathbf{u}'_r, \mathbf{v}'_r)$ at a time while holding all others fixed. We define the residual when excluding component $r$:
\begin{equation}
\mathbf{E}^{(r)} = \mathbf{W}' - \sum_{j \neq r} \sigma'_j \mathbf{u}'_j \mathbf{v}'^\top_j = \mathbf{W}' - \mathbf{W}'_k + \sigma'_r \mathbf{u}'_r \mathbf{v}'^\top_r
\end{equation}

The optimization for component $r$ becomes:
\begin{equation}
\min_{\sigma'_r, \mathbf{u}'_r, \mathbf{v}'_r} \sum_{i,j} (1 + \delta \cdot i_{ij}) \bigl(\mathbf{E}^{(r)}_{ij} - \sigma'_r u'_{r,i} v'_{r,j}\bigr)^2
\end{equation}

\paragraph{Derivation of $\mathbf{v}'_r$ Update}

Fixing $\mathbf{u}'_r$ and $\sigma'_r$, we minimize with respect to $\mathbf{v}'_r$. The objective for element $v'_{r,j}$ is:
\begin{equation}
\mathcal{L}_{v'_{r,j}} = \sum_{i} (1 + \delta \cdot i_{ij}) \bigl(\mathbf{E}^{(r)}_{ij} - \sigma'_r u'_{r,i} v'_{r,j}\bigr)^2
\end{equation}

Taking the derivative and setting to zero:
\begin{equation}
\frac{\partial \mathcal{L}_{v'_{r,j}}}{\partial v'_{r,j}} = -2 \sum_{i} (1 + \delta \cdot i_{ij}) \bigl(\mathbf{E}^{(r)}_{ij} - \sigma'_r u'_{r,i} v'_{r,j}\bigr) \sigma'_r u'_{r,i} = 0
\end{equation}

Solving for $v'_{r,j}$:
\begin{align}
\sigma'_r \sum_{i} (1 + \delta i_{ij}) \mathbf{E}^{(r)}_{ij} u'_{r,i} &= \sigma'^2_r v'_{r,j} \sum_{i} (1 + \delta i_{ij}) (u'_{r,i})^2 \\[6pt]
v'_{r,j} &= \frac{\sum_{i} (1 + \delta i_{ij}) \mathbf{E}^{(r)}_{ij} u'_{r,i}}{\sigma'_r \sum_{i} (1 + \delta i_{ij}) (u'_{r,i})^2}
\end{align}

In matrix notation, this yields the closed-form update:
\begin{equation}
\mathbf{v}'_r = \left( \frac{\mathbf{u}'^\top_r \bigl((\mathbf{1} + \delta \cdot \mathbf{I}) \odot \mathbf{E}^{(r)}\bigr)}{\sigma'_r \cdot (\mathbf{u}'^2_r)^\top (\mathbf{1} + \delta \cdot \mathbf{I})} \right)^\top
\end{equation}
where the division is element-wise and $\mathbf{u}'^2_r$ denotes element-wise squaring. The transpose ensures $\mathbf{v}'_r$ is a column vector: the numerator computes a row vector where the $j$-th element is the influence-weighted inner product of $\mathbf{u}'_r$ with the $j$-th column of $\mathbf{E}^{(r)}$, and the denominator is a row vector of normalization factors for each column. This corresponds to Equation~\eqref{eq:update_v} in the main text.

\paragraph{Derivation of $\mathbf{u}'_r$ Update}

Fixing $\mathbf{v}'_r$, we minimize with respect to $\mathbf{u}'_r$ and $\sigma'_r$ jointly. The objective for the $i$-th row is:
\begin{equation}
\mathcal{L}_{i} = \sum_{j} (1 + \delta \cdot i_{ij}) \bigl(\mathbf{E}^{(r)}_{ij} - \sigma'_r u'_{r,i} v'_{r,j}\bigr)^2
\end{equation}

Let $\tilde{u}'_{r,i} = \sigma'_r u'_{r,i}$ denote the unnormalized left singular vector element. Taking the derivative with respect to $\tilde{u}'_{r,i}$ and setting to zero:
\begin{equation}
\frac{\partial \mathcal{L}_{i}}{\partial \tilde{u}'_{r,i}} = -2 \sum_{j} (1 + \delta \cdot i_{ij}) \bigl(\mathbf{E}^{(r)}_{ij} - \tilde{u}'_{r,i} v'_{r,j}\bigr) v'_{r,j} = 0
\end{equation}

Solving for $\tilde{u}'_{r,i}$:
\begin{align}
\sum_{j} (1 + \delta \cdot i_{ij}) \mathbf{E}^{(r)}_{ij} v'_{r,j} &= \sum_{j} (1 + \delta \cdot i_{ij}) \tilde{u}'_{r,i} (v'_{r,j})^2 \\
\tilde{u}'_{r,i} &= \frac{\sum_{j} (1 + \delta \cdot i_{ij}) \mathbf{E}^{(r)}_{ij} v'_{r,j}}{\sum_{j} (1 + \delta \cdot i_{ij}) (v'_{r,j})^2}
\end{align}

In matrix notation:
\begin{equation}
\tilde{\mathbf{u}}'_r = \frac{\bigl((\mathbf{1} + \delta \cdot \mathbf{I}) \odot \mathbf{E}^{(r)}\bigr) \mathbf{v}'_r}{(\mathbf{1} + \delta \cdot \mathbf{I}) (\mathbf{v}'^2_r)}
\end{equation}
where the division is element-wise and $\mathbf{v}'^2_r$ denotes element-wise squaring. Note that simplifying by canceling $\mathbf{v}'_r$ between numerator and denominator is not possible: since these are matrix-vector products, each resulting element $i$ involves distinct summations $\sum_j (\cdot) v'_{r,j}$ and $\sum_j (\cdot) (v'_{r,j})^2$, mixing different components of $\mathbf{v}'_r$. This corresponds to Equation~\eqref{eq:update_u} in the main text.

Finally, we extract the singular value and normalized left singular vector:
\begin{equation}
\sigma'_r = \|\tilde{\mathbf{u}}'_r\|_2, \qquad \mathbf{u}'_r = \frac{\tilde{\mathbf{u}}'_r}{\sigma'_r}
\end{equation}

This normalization ensures $\|\mathbf{u}'_r\|_2 = 1$, maintaining the standard SVD convention.

\paragraph{Convergence Properties and Justification for a Single ALS Sweep}

We begin with two formal statements that characterize the convergence behavior of the \AIR{} solution. The first establishes the optimality of the activation-aware initialization in the special case $\delta=0$; the second establishes monotone descent of the influence-weighted objective along the ALS sweep.

\begin{proposition}[Optimality of the activation-aware initialization]\label{prop:init}
At $\delta=0$, the influence-weighted objective $\mathcal{L}_{\text{act,infl}}$ (Eq.~\ref{eq:hybrid_loss}) reduces to the activation-aware objective $\mathcal{L}_{\text{act}}$ (Eq.~\ref{eq:act_loss}). The initialization $(\mathbf{U}'_k, \boldsymbol{\Sigma}'_k, \mathbf{V}'_k) \leftarrow \mathrm{SVD}(\mathbf{W}', k)$ is the global minimizer of $\mathcal{L}_{\text{act}}$ over all rank-$k$ factorizations.
\end{proposition}

\begin{proof}
At $\delta=0$, $(\mathbf{1} + \delta \cdot \mathbf{I}) = \mathbf{1}$, so $\mathcal{L}_{\text{act,infl}} = \mathcal{L}_{\text{act}} \odot \mathbf{1} = \mathcal{L}_{\text{act}}$. The scalar objective minimized is $\sum_{ij}[\mathcal{L}_{\text{act}}]_{ij} = \|\mathbf{W}' - \mathbf{U}'_k\boldsymbol{\Sigma}'_k\mathbf{V}'^\top_k\|_F^2$. The Eckart--Young theorem~\cite{eckart1936approximation} guarantees that the truncated SVD provides the global minimum of this Frobenius-norm reconstruction error over all rank-$k$ matrices. Since the profiled weight $\mathbf{W}'$ is fixed and the optimization is over its rank-$k$ factorizations, $\mathrm{SVD}(\mathbf{W}', k)$ is the global minimizer.
\end{proof}

For $\delta > 0$, the joint optimization over $(\mathbf{U}'_k, \boldsymbol{\Sigma}'_k, \mathbf{V}'_k)$ is non-convex due to the product structure $\mathbf{U}'_k \boldsymbol{\Sigma}'_k \mathbf{V}'^\top_k$. Standard global-optimality arguments do not apply, but a coordinate-wise descent guarantee does:

\begin{proposition}[Monotone descent of the ALS sweep]\label{prop:monotone}
Let $({\mathbf{U}'_k}^{(0)}, {\boldsymbol{\Sigma}'_k}^{(0)}, {\mathbf{V}'_k}^{(0)}) = \mathrm{SVD}(\mathbf{W}', k)$. Define the sequence of iterates obtained by sweeping $r = k-1, k-2, \ldots, 0$ and applying Eq.~\ref{eq:update_v} followed by Eq.~\ref{eq:update_u} at each step. Then the influence-weighted objective $\mathcal{L}_{\text{act,infl}}$ is non-increasing along this sequence: for every step $r$,
\[
\mathcal{L}_{\text{act,infl}}\bigl((\mathbf{U}'_k, \boldsymbol{\Sigma}'_k, \mathbf{V}'_k)_{\text{after step }r}\bigr) \;\leq\; \mathcal{L}_{\text{act,infl}}\bigl((\mathbf{U}'_k, \boldsymbol{\Sigma}'_k, \mathbf{V}'_k)_{\text{before step }r}\bigr),
\]
with equality if and only if the corresponding component is already at the coordinate-wise optimum.
\end{proposition}

\begin{proof}
At step $r$, the update for $\mathbf{v}'_r$ holds $\mathbf{U}'_k$, $\boldsymbol{\Sigma}'_k$, and all $\mathbf{v}'_j$ with $j \neq r$ fixed. Under this restriction, $\mathcal{L}_{\text{act,infl}}$ is a strictly convex quadratic in $\mathbf{v}'_r$: the Hessian is the diagonal matrix with $j$-th entry $\sigma'^2_r \sum_i (1 + \delta \cdot i_{ij})(u'_{r,i})^2 > 0$ (positive whenever $\sigma'_r \neq 0$ and the influence weights are non-negative, both of which hold by construction). Setting the gradient to zero yields the unique minimizer given by Eq.~\ref{eq:update_v}; the derivation is carried out element-wise above. Hence the new iterate has objective value no greater than the previous one. The same argument applies to the joint update of $(\mathbf{u}'_r, \sigma'_r)$ via Eq.~\ref{eq:update_u}: with all other components fixed, the objective is a strictly convex quadratic in the unnormalized vector $\tilde{\mathbf{u}}'_r = \sigma'_r \mathbf{u}'_r$, and Eq.~\ref{eq:update_u} provides its unique minimizer. The subsequent normalization $\sigma'_r = \|\tilde{\mathbf{u}}'_r\|_2$, $\mathbf{u}'_r = \tilde{\mathbf{u}}'_r / \sigma'_r$ is a re-parameterization of the same point in the rank-1 component $\sigma'_r \mathbf{u}'_r \mathbf{v}'^\top_r$ and leaves the objective value unchanged. Composing the two updates per step over the $k$ steps of the sweep yields the stated monotone descent.
\end{proof}

\begin{remark}[Scope of the guarantee]
Proposition~\ref{prop:monotone} establishes monotone descent of the objective along the sweep, not convergence of the iterate sequence to a global minimum: the underlying problem is non-convex in the joint factorization $\mathbf{U}'_k \boldsymbol{\Sigma}'_k \mathbf{V}'^\top_k$. This is the standard guarantee for block-coordinate descent on non-convex objectives and matches the convergence statements available for related closed-form one-pass methods such as GPTQ~\cite{Frantar2023GPTQ}. Combined with Proposition~\ref{prop:init}, however, the result has practical force: \AIR{} starts at the global minimum of the unperturbed objective $\mathcal{L}_{\text{act}}$, and every subsequent step provably does not worsen the perturbed objective $\mathcal{L}_{\text{act,infl}}$.
\end{remark}

As shown in Table~\ref{tab:als_ablations} (Appendix~\ref{subsec:als_dynamics_ext}), the backward iteration order (from $r = k-1$ to $r = 0$) matters: it protects the dominant singular components (corresponding to larger singular values, i.e., smaller $r$ indices) from distortion, as the largest perturbations occur in the first iterations of the ALS when individual rank updates can be largest.

A natural question is whether multiple ALS sweeps would further improve compression quality. Our experiments show that while $\mathcal{L}_{\text{act,infl}}$ can continue to decrease marginally in subsequent sweeps, downstream perplexity does not improve. We identify three reasons why a single sweep suffices:

\textit{(1) Strong initialization from the unperturbed optimum.}
By Proposition~\ref{prop:init}, \AIR{} initializes from the global minimum of the activation-aware objective $\mathcal{L}_{\text{act}}$ (Eq.~\ref{eq:act_loss}), and by Proposition~\ref{prop:monotone} every subsequent ALS step does not increase the perturbed objective $\mathcal{L}_{\text{act,infl}}$ (Eq.~\ref{eq:hybrid_loss}). Since $\mathcal{L}_{\text{act,infl}}$ is a bounded multiplicative perturbation of $\mathcal{L}_{\text{act}}$, with $\mathbf{I}$ normalized to unit mean, a single coordinate descent sweep from this optimum captures the dominant correction. Subsequent sweeps would only adjust for second-order interactions between rank updates, which are small because the rank components are approximately orthogonal at initialization (being exact singular vectors of $\mathbf{W}'$). Empirically, these small additional adjustments slightly degrade downstream perplexity (Table~\ref{tab:als_ablations}) even as $\mathcal{L}_{\text{act,infl}}$ continues to decrease, consistent with mild overfitting to the local proxy on the calibration set: $\mathcal{L}_{\text{act,infl}}$ is computed from a finite calibration sample and serves as a surrogate for the end-task loss, so further refinement of the proxy beyond its dominant correction need not transfer to held-out perplexity.

\textit{(2) Structural analogy to GPTQ.}
Our single-sweep strategy parallels GPTQ~\citep{Frantar2023GPTQ}, which processes weight columns in a single left-to-right pass: at each column, it applies a closed-form optimal update (derived from inverse Hessian information), propagates the residual to remaining columns, and never revisits. Despite this single-pass design, and without formal convergence guarantees beyond empirical validation, GPTQ achieves near-optimal quantization up to 175B parameters. The parallel is direct: both methods decompose a matrix optimization into sequential subproblems over individual components (columns for GPTQ, ranks for \AIR{}), solve each exactly via closed-form expressions, and propagate the residual in a single sweep. Our setting is arguably more favorable, as \AIR{} starts from a provably optimal initialization rather than from the original pretrained weights.

\textit{(3) Empirical confirmation.}
We verified that additional ALS sweeps do not improve downstream perplexity, suggesting that the single sweep already captures all compression-relevant information from the influence signal.

\subsection{ALS dynamics: extended ablations and relevance properties}\label{subsec:als_dynamics_ext}

This appendix contains the relevance-property analysis that the main-body Sec.~\ref{subsec:als_dynamics} summarizes (the ALS ablation tables, Tables~\ref{tab:influence_metric_ablation} and~\ref{tab:als_ablations}, are reported in the main body).

\paragraph{Spatial relevance patterns.} Figure~\ref{fig:layer_block_14} provides a spatial view of the same observation made in main-text Figure~\ref{fig:scatter_correlation} (the disconnect between native-space magnitudes and attributed relevance), at the matrix level. The corresponding broader maps appear in Appendix~\ref{sec:appendix_spatial_patterns}: regions of high relevance often do not coincide with regions of high activation or weight magnitude.

\begin{figure}[h!]
\centering
\includegraphics[width=0.85\linewidth,trim={50 530 220 70},clip]{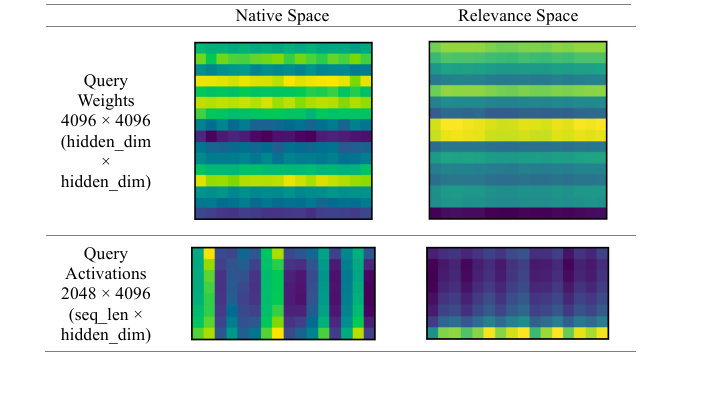}
\caption{Spatial heatmaps in native vs.\ relevance space, shown for the attention query weight and activation matrices in layer block 10 of LLaMA 7B. Elements are accumulated over 256$\times$256 blocks.}
\label{fig:layer_block_14}
\end{figure}

\subsection{Efficiency for token-generation through implementation}\label{subsec:implementation_efficiency}

SVD-based compression achieves proportional reductions in parameter count, FLOPs, and memory footprint (main-text Table~\ref{tab:compression_metrics_correlation}). However, for autoregressive token generation, the two deployment-critical metrics, \emph{peak GPU memory} and \emph{per-token latency}, do not immediately benefit in proportion. Instead, they require specific forward-pass optimizations and further depend on the operating regime (batch size, sequence length). This section analyzes these challenges, examines efficiency claims in prior work, and presents the optimizations that make SVD-based compression effective for token generation.

\subsubsection{Why standard SVD compression does not reduce peak memory or latency}\label{subsubsec:why_no_gains}

\paragraph{KV cache: no memory reduction without optimization.}
During autoregressive generation, the KV cache, not the weight matrices, dominates GPU memory for long sequences. Standard SVD-compressed models compute the two-stage projection $\mathbf{U}_k (\mathbf{V}_k^\top \mathbf{x})$ and expand the result back to full dimension before entering the KV cache, so the cache stores keys and values at the \emph{original} $n_{\text{kv}} \cdot d_h$ dimension per token. The weight compression thus yields zero cache reduction, and peak GPU memory remains dominated by the uncompressed KV cache.

\paragraph{Two-matmul overhead: no latency reduction.}
Replacing a single matrix multiplication $\mathbf{y} = \mathbf{W}\mathbf{x}$ with $\mathbf{y} = \mathbf{U}_k(\mathbf{V}_k^\top \mathbf{x})$ reduces the MAC count from $mn$ to $k(m+n)$, a factor $\frac{k(m+n)}{mn}$ that mirrors the parameter ratio. However, on modern GPUs, autoregressive decoding processes inputs of shape $(b, 1, d)$, placing inference in a \emph{memory-bandwidth-bound} regime where time is dominated by reading weights from HBM, not arithmetic. In this regime, two smaller weight reads are not necessarily faster than one contiguous read due to kernel launch overhead, reduced cache locality from intermediate materialization, and lower SM occupancy when individual GEMMs are too small to saturate the GPU. The theoretical FLOP reduction therefore does not translate proportionally, or at all, into wall-clock speedup.

This is consistent across the literature, though rarely stated explicitly. ACIP~\cite{genzel2025choosemodelsizecompression} is a notable exception: LLaMA-7B compressed to 40\% achieves only 1595 tokens/s vs.\ 2448 tokens/s uncompressed, a 35\% \emph{slowdown} despite 58\% fewer FLOPs. None of the SVD-based methods surveyed (ASVD~\cite{yuan2024asvdactivationawaresingularvalue}, SVD-LLM~\cite{svd_llm}, SVD-LLM V2~\cite{wang2025svdllmv2optimizingsingular}) provide fused CUDA kernels to bridge this gap.

\subsubsection{Scrutiny of efficiency claims in prior work}\label{subsubsec:misleading_claims}

\paragraph{Inflated throughput metrics.}
SVD-LLM's \texttt{eff\_eval()} computes throughput as $b \cdot l_{\max} / t_{\text{wall}}$, counting the maximum generation length regardless of early EOS termination. Compressed models produce EOS earlier due to degraded quality, so the metric counts \emph{phantom tokens} never generated, systematically inflating speedups for more compressed models.
On TinyLlama 1.1B at 60\% parameter rate, this metric reports a 2.65$\times$ speedup (322 vs.\ 855~tok/s). However, the base model generates 726/1024 tokens on average while the compressed model generates only 200/1024. Corrected throughput: 229~tok/s (base) vs.\ 167~tok/s (compressed); the compressed model is 27\% \emph{slower}, consistent with the memory-bandwidth-bound analysis above.

\paragraph{Unsubstantiated KV cache compression.}
Both ASVD~\cite{yuan2024asvdactivationawaresingularvalue} and SVD-LLM~\cite{svd_llm} discuss KV cache compression in their papers, but their implementations differ from the claims.
ASVD's \texttt{compress\_kv\_cache} flag restricts rank allocation search to K and V projections, but the \texttt{SVDLinear} module still computes $\mathbf{U}_k (\mathbf{V}_k^\top \mathbf{x})$ at full output dimension with no modified attention module, yielding no cache reduction.
SVD-LLM provides a dedicated module (\texttt{svd\_llama\_kvcache.py}) caching low-rank intermediates for \emph{both} keys and values, but their strategy re-expands both to full dimension before attention at every step, trading cache memory for per-step computation overhead (acknowledged in their limitations section). The published code contains errors (incorrect RoPE function signatures) and multiple \texttt{TODO} annotations, indicating a non-functional prototype.

\subsubsection{Our proposed optimizations}\label{subsubsec:our_optimizations}

\paragraph{Common practices.}
All SVD-based methods propagate hidden states sequentially through the decomposed factors rather than reconstructing the full weight matrix, i.e.\ $\mathbf{y} = \mathbf{U}_k (\boldsymbol{\Sigma}_k (\mathbf{V}_k^\top \mathbf{x}))$ instead of the naive $\mathbf{y} = (\mathbf{U}_k \boldsymbol{\Sigma}_k \mathbf{V}_k^\top) \mathbf{x}$. This avoids materializing the full $m \times n$ matrix and reduces the MAC count from $mn$ to $k(m+n)$, a factor matching the parameter rate. Additionally, the diagonal $\boldsymbol{\Sigma}_k$ is absorbed via $\mathbf{A} = \mathbf{U}_k \sqrt{\boldsymbol{\Sigma}_k}$, $\mathbf{B} = \sqrt{\boldsymbol{\Sigma}_k} \, \mathbf{V}_k^\top$, eliminating the middle matmul and reducing the forward pass to two matrix multiplications: $\mathbf{y} = \mathbf{A}(\mathbf{B}\mathbf{x})$. The balanced absorption distributes magnitudes evenly, which also benefits downstream quantization~\cite{yuan2024asvdactivationawaresingularvalue, svd_llm}.

\paragraph{Beyond common practices: KV cache and buffer optimizations.}\label{par:kv_cache_optimization}
We introduce three optimizations that target the KV cache bottleneck. All three route caching through an internal path that bypasses the standard HuggingFace KV cache, with a shared prerequisite structure.

\textit{(O1) RoPE pre-application (prerequisite).}
Rotary Position Embeddings are applied to keys immediately during cache writes, storing $\mathbf{k}_t = \mathbf{R}(t)\,\mathbf{U}^{(k)}_k(\mathbf{V}^{(k)T}_k \mathbf{x}_t)$. Without this, the standard code path attempts to re-apply RoPE to the full accumulated cache with position IDs corresponding only to the current token, producing incorrect key states. O1 is a prerequisite for O2 and O3.

\textit{(O2) Fused low-rank value caching.}
Values undergo no position-dependent transformation, so the intermediate $\mathbf{v}_t = \mathbf{V}^{(v)T}_k \mathbf{x}_t \in \mathbb{R}^{r_v}$ (with $r_v \ll n_{\text{kv}} \cdot d_h$) can be cached directly. The expansion through $\mathbf{U}^{(v)}_k$ is deferred and fused into the attention computation:
\begin{equation}
\text{Attn}(\mathbf{A}, \mathbf{v}_{1:s}) = \underbrace{\Bigl(\sum_{t=1}^{s} a_t \cdot \mathbf{v}_t\Bigr)}_{\in\, \mathbb{R}^{r_v}} \cdot \mathbf{U}^{(v)T}_k,
\end{equation}
where $\mathbf{A} \in \mathbb{R}^{n_h \times 1 \times s}$ are the attention weights. The weighted sum operates in $r_v$-dimensional space; only the final result is projected to full dimension. This reduces the per-token value cache from $n_{\text{kv}} \cdot d_h$ to $r_v$ elements.

Keys cannot benefit from analogous caching: RoPE is a position-dependent rotation $\mathbf{R}(n)$ in the full $d_h$-dimensional space, so low-rank key caching would require re-expanding and re-applying RoPE to all $s$ cached keys at every step ($\mathcal{O}(s \cdot r_k \cdot d_h)$ per layer). Absorbing RoPE into the projection is equally infeasible, as $\mathbf{R}(n) \cdot \mathbf{U}^{(k)}_k$ produces a position-specific $(d_h \times r_k)$-matrix requiring more storage than full-dimension keys.

\begin{center}
\scriptsize
\begin{tabular}{lcc}
\toprule
& \textsc{Standard} & \textsc{Optimized} \\
\midrule
Keys & $n_\text{kv}\cdot d_h$ & $n_\text{kv}\cdot d_h$ \\
Values & $n_\text{kv}\cdot d_h$ & $r_v$ \\
\midrule
Total & $2\cdot n_\text{kv}\cdot d_h$ & $n_\text{kv}\cdot d_h + r_v$ \\
\bottomrule
\end{tabular}
\end{center}

\textit{(O3) Pre-allocated cache buffers.}
Cache tensors are allocated once at the maximum sequence length and filled via in-place writes, avoiding per-token dynamic memory allocations. The optimized path also precomputes $\mathbf{U}^{(v)}_k$ reshaped per attention head, avoiding repeated reshape and repeat operations at each decode step.

\subsubsection{Ablation of forward-pass optimizations}\label{subsubsec:kv_ablation}

Table~\ref{tab:kv_ablation} reports the impact of each optimization on peak GPU memory and per-token latency during autoregressive generation (batch size 64, sequence length 512, LLaMA 7B at 100\% and 60\% parameter rate on a single A100 GPU), as rates relative to the base model. The first row is the compressed model on the default HuggingFace cache path; rows below add O1 (RoPE pre-application) as prerequisite to the alternative internal-cache path.

\begin{table}[t!]
\centering
\caption{Ablation of forward-pass optimizations on peak GPU memory and per-token latency. Each O1-prefixed row adds optimizations on top of the internal-cache path; the first row is the default HuggingFace cache path. The base model at BS=64, SL=512 (29.34\,GB peak memory, 1540\,$\mu$s/token) serves as reference (100\%).}
\label{tab:kv_ablation}
{\footnotesize
\renewcommand{\arraystretch}{1.05}
\begin{tabular}{l|cc|cc}
\toprule
& \multicolumn{2}{c|}{\textsc{Peak GPU Mem.\ Rate}} & \multicolumn{2}{c}{\textsc{Latency Rate}} \\
\cmidrule(lr){2-3}\cmidrule(lr){4-5}
\textsc{Param. Rate} & 100\% & 60\% & 100\% & 60\% \\
\midrule
Compressed, w/o opt. & 100.3\% & 84.0\% & 125.5\% & 117.3\% \\
O1                   & 100.4\% & 84.0\% & 124.5\% & 118.9\% \\
O1+O3                & 100.4\% & 84.0\% & 50.1\%  & 50.5\%  \\
O1+O2                & 86.3\%  & 64.3\% & 91.2\%  & 76.6\%  \\
\midrule
O1+O2+O3             & 86.4\%  & 64.2\% & 46.6\%  & 44.2\%  \\
\bottomrule
\end{tabular}}
\end{table}

O1 alone (switching from the default HuggingFace cache path to the internal cache path with RoPE pre-application) does not by itself improve latency, and at 60\% parameter rate slightly worsens it (latency rate $117.3\%\rightarrow118.9\%$), because the internal caching path lacks the optimized memory access patterns of HuggingFace's DynamicCache while providing no memory benefit at full dimension.
O2 provides the dominant memory reduction (from $84.0\%$ to $64.3\%$ at 60\% param rate, comparing O1 to O1+O2), confirming value cache compression as the primary driver.
O3 provides the dominant latency improvement (from $118.9\%$ to $50.5\%$ at 60\% param rate, comparing O1 to O1+O3), confirming that eliminating per-token dynamic allocations is the main driver of latency.
The full combination O1+O2+O3 achieves the best overall trade-off, demonstrating that memory and latency optimizations are complementary.

\paragraph{Implications.}
The absence of inference speedup under standard implementations does not diminish SVD-based compression: reductions in parameter count, memory footprint, and peak GPU allocation (Table~\ref{tab:compression_metrics_correlation}) matter for memory-constrained deployment, larger batch sizes, or combination with quantization. However, translating FLOP reductions into latency improvements requires systems-level co-design, as shown by our O1--O3 optimizations. In contrast to the unsubstantiated KV cache claims of prior work (\S\ref{subsubsec:misleading_claims}), our approach (i) caches values in low-rank form and \emph{avoids re-expansion} through fused attention, (ii) maintains keys at full dimension with RoPE correctly applied, and (iii) implements optimizations (weight precomputation, pre-allocated buffers) that make the strategy deployment-ready. These considerations apply to all SVD-based methods and are complementary to approximation quality, the focus of \AIR{}.

\subsubsection{End-to-end results: peak GPU memory and per-token latency}\label{subsubsec:end_to_end_efficiency}

\begin{figure}[t!]
    \centering
    \includegraphics[width=\linewidth]{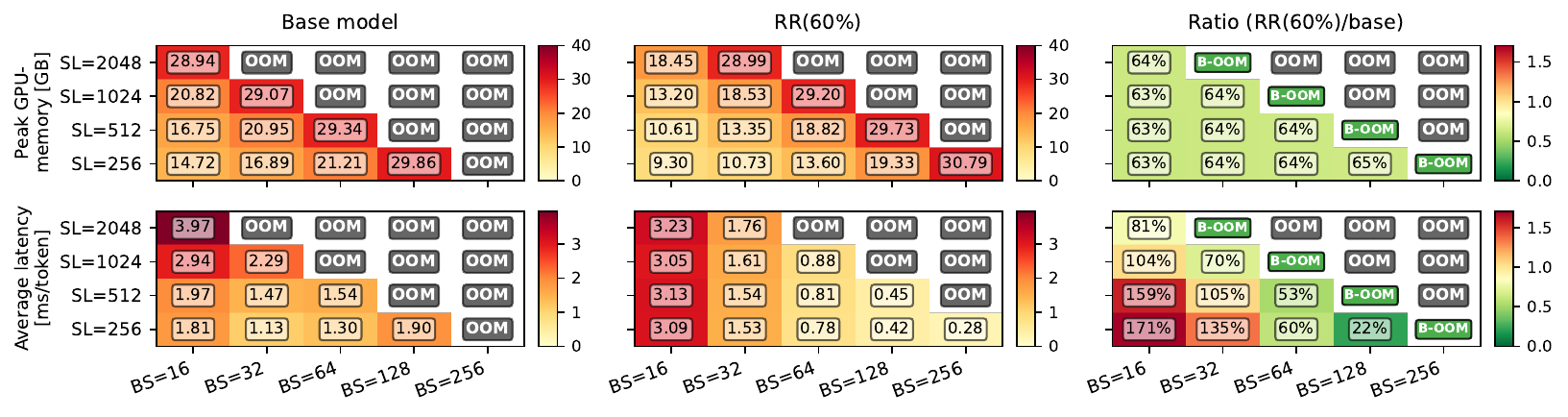}
    \caption{Extended runtime efficiency analysis (full version of Figure~\ref{fig:peak_mem_throughput_rate} in the main body): peak GPU-memory in GB (top row) and average per-token latency in ms/token (bottom row), reported in absolute units for the uncompressed base model (left column) and \AIR{}(60\%) (middle column), alongside the cell-wise ratio \AIR{}(60\%)/base (right column, reproduced from the main body). ``OOM'' marks cells that exceed the 40\,GB A100 memory budget; ``B-OOM'' marks cells where only the base model runs out of memory while \AIR{}(60\%) still fits.}
    \label{fig:peak_mem_throughput_rate_full}
\end{figure}

We now report end-to-end peak GPU memory and per-token latency for the full O1+O2+O3 \AIR{} pipeline against the uncompressed base model.

Figure~\ref{fig:peak_mem_throughput_rate_full} (the full version of the main-body Figure~\ref{fig:peak_mem_throughput_rate}, with absolute values for both the base model and \AIR{}(60\%)) presents peak GPU memory and average per-token latency for a scenario where a user prompts LLaMA 7B (compressed to 60\% parameter rate) with a query of 16 tokens and receives a response between 256 and 2048 tokens. The axes \textit{Tokens} and batch size \textit{BS} describe the data structure of the final response; since this structure builds up token by token during autoregressive generation, it also determines the model's cumulative input context at each decoding step and thus the peak memory and average per-token latency reported in the figure.

Peak GPU memory determines \emph{feasibility}: given a fixed GPU budget, it dictates which batch-size and sequence-length combinations can be served at all. Both the base model and \AIR{} scale their memory usage with batch size and sequence length at a similar rate, but \AIR{} operates at a lower level: the cell-wise ratio remains ${\sim}64\%$ throughout. This uniform offset means the base model hits the OOM boundary one diagonal earlier in the batch-size$\times$sequence-length grid, while \AIR{} can progress to the next larger configuration. For instance, at BS=64 with 2048 tokens, or BS=128 with 512 tokens, the base model runs out of memory while \AIR{}(60\%) completes.

The more deployment-relevant metric, however, is \emph{latency} (or, equivalently, throughput). In a real deployment, the reasoning is sequential: (1)~determine the available GPU memory budget (here, 40\,GB on an A100), (2)~estimate the expected response sequence length required to serve each user, and (3)~given these constraints, choose the batch size that minimizes per-token latency, equivalently, maximizes throughput, serving as many users as possible at the required response lengths. Peak GPU memory thus acts as a precondition that determines the feasible region in the batch-size$\times$sequence-length grid; latency then determines the optimal operating point within that region.

The cell-wise latency ratios (third column of Figure~\ref{fig:peak_mem_throughput_rate_full}) show two regimes. At large batch sizes and short sequences, \AIR{} achieves speedups (down to $22\%$ at BS=128, SL=256), as the reduced memory footprint alleviates bandwidth pressure and allows higher GPU utilization. At small batch sizes and short sequences, the two-matmul overhead of SVD-based compression dominates, and latency exceeds the base model's (up to $171\%$ at BS=16, SL=256).

A more deployment-relevant perspective emerges when we follow the practical strategy outlined above. Consider 256-token responses on a 40\,GB A100: for each model, we select the batch size that yields the lowest latency among those that fit in memory. The base model achieves its optimum at BS=64 with 0.81\,ms/tok; its latency \emph{increases} again at BS=128 (1.23\,ms/tok) due to memory pressure, whereas \AIR{}'s latency continues to decrease monotonically with batch size. \AIR{} fits BS=256 and reaches 0.29\,ms/tok, yielding a latency rate of $0.29/0.81 \approx 35\%$. This represents a different comparison from the cell-wise best of $22\%$ reported above: not a matched-configuration ratio, but the ratio between each model's respective best operating point given the same hardware constraints.

\subsection{Full perplexity, JS-divergence, and reasoning comparison}\label{subsec:full_perplexity_comparison}

Table~\ref{tab:perplexity_comparison_full} expands the condensed main-text Table~\ref{tab:perplexity_comparison} to all four parameter rates (80\%, 60\%, 40\%, 20\%), reports JS-divergence~\cite{jsd_khanal2024evaluatingimpactcompressiontechniques} alongside perplexity, and breaks down reasoning accuracy across all six tasks (OpenBookQA, ARC-E, WinoGrande, HellaSwag, PIQA, MathQA).

\begin{table*}[h!]
\centering
\caption{SVD-based compression of LLaMA 7B without enhancements: perplexity ($\downarrow$), JS-divergence vs.\ base ($\uparrow$, $0\%$ = identical, negative = divergence), and reasoning tasks ($\uparrow$) at varying parameter rates. Full version of main-text Table~\ref{tab:perplexity_comparison}.}
\label{tab:perplexity_comparison_full}
\resizebox{\textwidth}{!}{%
\scriptsize
\setlength{\tabcolsep}{3.5pt}
\begin{tabular}{l|l|cc|cc|ccccccc}
\toprule
    \textsc{Param.} & \textsc{Method} & \multicolumn{2}{c|}{\textsc{Perplexity} $\downarrow$} & \multicolumn{2}{c|}{\textsc{JS-divergence} $\uparrow$} & \multicolumn{7}{c}{\textsc{Reasoning tasks} $\uparrow$} \\
    \cmidrule(lr){3-4} \cmidrule(lr){5-6} \cmidrule(lr){7-13}
    \textsc{Rate} & & WikiText2 $\downarrow$ & C4 $\downarrow$ & WikiText2 $\uparrow$ & C4 $\uparrow$ & OpenB. $\uparrow$ & ArcE. $\uparrow$ & WinoG. $\uparrow$ & HellaS. $\uparrow$ & PIQA $\uparrow$ & MathQA $\uparrow$ & Average $\uparrow$ \\
\midrule
    \textcolor{gray}{100\%} & \textcolor{gray}{Base Model} & \textcolor{gray}{5.68} & \textcolor{gray}{7.34} & \textcolor{gray}{0\%} & \textcolor{gray}{0\%} & \textcolor{gray}{34\%} & \textcolor{gray}{75\%} & \textcolor{gray}{70\%} & \textcolor{gray}{57\%} & \textcolor{gray}{79\%} & \textcolor{gray}{27\%} & \textcolor{gray}{57\%} \\
\midrule
    & Vanilla SVD & 19438 & 16115 & -95\% & -94\% & 16.2\% & 27.4\% & 51.2\% & 26.3\% & 54.1\% & 21.3\% & 32.8\% \\
    & FWSVD & 22026 & 32048 & -96\% & -96\% & 15.4\% & 26.7\% & 49.3\% & 25.5\% & 53.3\% & 17.8\% & 31.3\% \\
    80\% & ASVD & 116 & 105 & -66\% & -62\% & 21.6\% & 41.3\% & 57.2\% & 32.6\% & 63.1\% & 20.9\% & 39.5\% \\
    & SVD-LLM(W) & 7.87 & 16.65 & -11\% & -23\% & 26.6\% & 64\% & 65.7\% & 43.3\% & 69.9\% & 23.2\% & 48.8\% \\
    \cmidrule{2-13}
    & \texttt{AIR} & \textbf{7.51} (\textcolor{DarkGreen}{$\downarrow$4.6\%}) & \textbf{14.24} (\textcolor{DarkGreen}{$\downarrow$14.5\%}) & \textbf{-8\%} (\textcolor{DarkGreen}{$\uparrow$27.3\%}) & \textbf{-19\%} (\textcolor{DarkGreen}{$\uparrow$17.4\%}) & 26.6\% & 66.3\% & 66.1\% & 45.3\% & 71.1\% & 24.1\% & \textbf{49.9\%} (\textcolor{DarkGreen}{$\uparrow$2.3\%}) \\
\midrule
    & Vanilla SVD & 52839 & 46630 & -98\% & -97\% & 15.4\% & 26\% & 52.7\% & 25.6\% & 52.6\% & 20.5\% & 32.1\% \\
    & FWSVD & 81838 & 111860 & -98\% & -98\% & 16.6\% & 24.8\% & 49.2\% & 25.8\% & 52.8\% & 17.3\% & 31.1\% \\
    60\% & ASVD & 4915 & 8103 & -92\% & -93\% & 13.4\% & 26.6\% & 49.4\% & 26.2\% & 53.6\% & 18.6\% & 31.3\% \\
    & SVD-LLM(W) & 13.81 & 56.33 & -25\% & -48\% & 20.6\% & 47.3\% & 55.4\% & 32.5\% & 61.2\% & 22.7\% & 40.0\% \\
    \cmidrule{2-13}
    & \texttt{AIR} & \textbf{11.27} (\textcolor{DarkGreen}{$\downarrow$18.4\%}) & \textbf{35.81} (\textcolor{DarkGreen}{$\downarrow$36.4\%}) & \textbf{-20\%} (\textcolor{DarkGreen}{$\uparrow$20.0\%}) & \textbf{-40\%} (\textcolor{DarkGreen}{$\uparrow$16.7\%}) & 19.2\% & 51.0\% & 59\% & 34.7\% & 63.7\% & 22.2\% & \textbf{41.6\%} (\textcolor{DarkGreen}{$\uparrow$4.0\%}) \\
\midrule
    & Vanilla SVD & 105082 & 105082 & -98\% & -97\% & 16.2\% & 26.4\% & 50.6\% & 25.5\% & 51.8\% & 20.3\% & 31.8\% \\
    & FWSVD & 134928 & 126754 & -98\% & -98\% & 16.4\% & 26.7\% & 49.2\% & 25.6\% & 52.9\% & 18.2\% & 31.5\% \\
    40\% & ASVD & 143631 & 134928 & -98\% & -98\% & 14.6\% & 25.4\% & 49.2\% & 25.6\% & 52.2\% & 21.1\% & 31.4\% \\
    & SVD-LLM(W) & 63.83 & 345 & -54\% & -72\% & 14.8\% & 30.1\% & 50.8\% & 27.3\% & 55\% & 21.5\% & 33.3\% \\
    \cmidrule{2-13}
    & \texttt{AIR} & \textbf{42.52} (\textcolor{DarkGreen}{$\downarrow$33.4\%}) & \textbf{277} (\textcolor{DarkGreen}{$\downarrow$19.7\%}) & \textbf{-48\%} (\textcolor{DarkGreen}{$\uparrow$11.1\%}) & \textbf{-69\%} (\textcolor{DarkGreen}{$\uparrow$4.2\%}) & 13.4\% & 32.2\% & 51.5\% & 27.5\% & 55.4\% & 21.7\% & \textbf{33.6\%} (\textcolor{DarkGreen}{$\uparrow$0.9\%}) \\
\midrule
    & Vanilla SVD & 196320 & 208981 & -98\% & -98\% & 16\% & 26.6\% & 50.2\% & 25.6\% & 54.1\% & 22\% & \textbf{32.4\%} (\textcolor{DarkGreen}{$\uparrow$2.2\%}) \\
    & FWSVD & 268337 & 208981 & -98\% & -98\% & 15.2\% & 25\% & 49.2\% & 25.5\% & 53.4\% & 19.1\% & 31.2\% \\
    20\% & ASVD & 24959 & 23447 & -96\% & -95\% & 11.4\% & 25.3\% & 51.6\% & 25.7\% & 52.9\% & 20.1\% & 31.2\% \\
    & SVD-LLM(W) & 854 & 8626 & -77\% & -84\% & 15.2\% & 25.8\% & 48.7\% & 26\% & 52\% & 20.2\% & 31.3\% \\
    \cmidrule{2-13}
    & \texttt{AIR} & \textbf{472} (\textcolor{DarkGreen}{$\downarrow$44.7\%}) & \textbf{2550} (\textcolor{DarkGreen}{$\downarrow$70.4\%}) & \textbf{-68\%} (\textcolor{DarkGreen}{$\uparrow$11.7\%}) & \textbf{-83\%} (\textcolor{DarkGreen}{$\uparrow$1.2\%}) & 15.2\% & 25.8\% & 49.7\% & 26.3\% & 52.9\% & 20.3\% & 31.7\% \\
\bottomrule
\end{tabular}
}
\end{table*}

\subsection{C4 as calibration data}

To verify that \AIR{}'s gains generalize beyond \textsc{WikiText-2} as the calibration source, we re-run the 60\% parameter-rate setting with C4 as the calibration set. \AIR{} reaches 19.47 perplexity vs.\ 34.17 for SVD-LLM(W), a 43\% improvement, confirming that the gains are not specific to the calibration distribution. Other parameter rates were not re-run for this calibration variant.

\subsection{Reasoning after LoRA Fine-Tuning}\label{subsec:reasoning_after_lora}

The main-text comparison of LoRA fine-tuning (Sec.~\ref{sec:experiments}, Table~\ref{tab:perplexity_comparison_lora}) reports perplexity. To verify that perplexity gains translate into downstream task performance, Table~\ref{tab:lora_reasoning} reports average zero-shot accuracy across the same reasoning benchmarks used in Table~\ref{tab:perplexity_comparison} (OpenBookQA, ARC-Easy, WinoGrande, HellaSwag, PIQA, MathQA). The accuracy ranking mirrors the perplexity ranking: \AIR{} + LoRA matches or outperforms SVD-LLM(W) + LoRA at every compression rate, with the gap widening at higher compression (\textbf{$\uparrow$5.3\%} at 20\% parameter rate). The advantage is largest where capacity is most constrained, the regime that matters most for deployment.

\begin{table}[t!]
\centering
\caption{Enhanced SVD-compression through LoRA fine-tuning for LLaMA 7B, measured by average zero-shot accuracy ($\uparrow$). Baseline average accuracy: 57\%.}
\label{tab:lora_reasoning}
{\footnotesize
\renewcommand{\arraystretch}{1.05}
\begin{tabular}{l|cccc}
\toprule
\textsc{Param. rate} & 80\% & 60\% & 40\% & 20\% \\
\midrule
SVD-LLM(W) + LoRA & \textbf{53.5} (\textcolor{DarkGreen}{$\uparrow$0.4\%}) & 47.6 & 40.5 & 32.3 \\
\midrule
\texttt{AIR + LoRA}
& 53.3
& \textbf{48.0} (\textcolor{DarkGreen}{$\uparrow$0.8\%})
& \textbf{40.8} (\textcolor{DarkGreen}{$\uparrow$0.7\%})
& \textbf{34.0} (\textcolor{DarkGreen}{$\uparrow$5.3\%}) \\
\bottomrule
\end{tabular}}
\end{table}

\subsection{Full retraining on the calibration set: a complementary end-to-end ablation}\label{subsec:retraining_discussion}\label{subsec:retraining}

The main-text Table~\ref{tab:perplexity_comparison_lora} compares \AIR{} and the prior state of the art under the most common end-to-end enhancement, LoRA fine-tuning on external data. This appendix probes a second end-to-end enhancement, \emph{full retraining of the compressed model on the calibration set itself}, and reports the cross-product with LoRA fine-tuning. The motivation is twofold. First, full retraining is rarely considered as a deployable post-compression step; we argue below why it is, in fact, deployable. Second, and more importantly for the central claim of the paper (the additivity reading of the design space), full retraining is the most exhaustive end-to-end use of the backward signal available to us, so the question of whether \AIR{}'s local backward-signal proxy still contributes once full retraining has been applied is a more demanding test than the LoRA comparison alone can provide.

\paragraph{Why full retraining qualifies as a deployable enhancement.}
A natural question is whether, given that \AIR{} already leverages the backward pass (and thus model-output information), one should simply perform \emph{full post-training} on the calibration data directly, foregoing the influence-metric approximation altogether.
Existing work rarely explores this option, presumably assuming it to be prohibitively expensive for LLMs and potentially harmful due to catastrophic forgetting.
We argue, however, that both concerns are overstated in the SVD-compression setting: a 7B-parameter model converges within 5--10 epochs on 256 calibration samples (WikiText-2) in about 12 minutes on an A100 GPU, and catastrophic forgetting is less of a concern after SVD truncation, which itself already disrupts the original weight structure.

We also reconsider a common categorization of post-compression methods by the binary criterion of \emph{whether a parameter update is applied or not}.
A more useful distinction is \emph{how much additional data and compute are required, and whether the contribution is complementary to other enhancements}.
Under this lens, full retraining on the same small calibration set uses no more data than the compression step itself.

\paragraph{Compositionality with \AIR{} and LoRA.}
Table~\ref{tab:perplexity_comparison_retraining} reports the cross-product of \AIR{} (vs.\ SVD-LLM) with the four enhancement combinations \{Default, LoRA on Alpaca, Retrain on WikiText-2, Retrain + LoRA\} at 40\% and 20\% parameter rates. At the rates we have measured, each of the three mechanisms (the local influence proxy through \AIR{}, full retraining on the calibration set, LoRA fine-tuning on Alpaca) appears to provide an independent gain on top of the others rather than substituting for them, and the three-way combination is the lowest-perplexity column for both \AIR{} and SVD-LLM. We treat this as a working hypothesis to be sharpened by filling the remaining cells (80\% and 60\%) and extending the comparison to additional model families; if it holds, it suggests that the two appearances of the backward signal in compression, as a per-element local influence score and as a parameter-update direction, are complementary contributions rather than redundant ones.

\begin{table}[t!]
\centering
\caption{Enhanced SVD-compression through retraining on given calibration data and/or fine-tuning on other data, measured by perplexity ($\downarrow$).}
\label{tab:perplexity_comparison_retraining}
{\footnotesize
\renewcommand{\arraystretch}{1.05}
\begin{tabular}{l|cccc}
\toprule
& \makecell{Default}
& \makecell{LoRA\\(Alpaca)}
& \makecell{Retrain\\(WikiText-2)}
& \makecell{Retrain(WikiText-2)\\+ LoRA(Alpaca)} \\
\midrule
SVD-LLM(W)(40\%)
& 63.8 & 16.82 & 13.59 & 13.22 \\
\texttt{AIR(40\%)}
& \textbf{42.5} (\textcolor{DarkGreen}{$\downarrow$33.4\%})
& \textbf{14.5} (\textcolor{DarkGreen}{$\downarrow$13.8\%})
& \textbf{12.77} (\textcolor{DarkGreen}{$\downarrow$6.0\%})
& \textbf{12.25} (\textcolor{DarkGreen}{$\downarrow$7.3\%}) \\
\midrule
SVD-LLM(W)(20\%)
& 854 & 114 & 98.9 & 94.76 \\
\texttt{AIR(20\%)}
& \textbf{472} (\textcolor{DarkGreen}{$\downarrow$44.7\%})
& \textbf{46.9} (\textcolor{DarkGreen}{$\downarrow$58.9\%})
& \textbf{47.4} (\textcolor{DarkGreen}{$\downarrow$52.1\%})
& \textbf{38.35} (\textcolor{DarkGreen}{$\downarrow$59.5\%}) \\
\bottomrule
\end{tabular}
}
\end{table}

\subsection{Sentence-completion after LoRA fine-tuning: A qualitative analysis of \AIR{}+LoRA vs.\ SVD-LLM(W)+LoRA}\label{subsec:sentence_completion}

Since LLaMA 7B is not fine-tuned as a chatbot, we decided to test for sentence completion rather than replying as an assistant to a user's question. Note that sentence completion is inherently open-ended; the model predicts plausible continuations rather than answering specific questions. As such, we consider topic shifts to be minor issues rather than indicators of model degradation, since the training data (books, articles, etc.) often contains such transitions that might be coherent in their original context. The outputs reported below are produced by the LoRA-fine-tuned variants of both methods (SVD-LLM(W)+LoRA and \AIR{}+LoRA, matching the setting of Table~\ref{tab:perplexity_comparison_lora}); for brevity we refer to them simply as SVD-LLM(W) and \AIR{} throughout the discussion.

Both prompts show a consistent pattern. At 80\% parameter rate, both SVD-LLM(W) and \AIR{} produce fluent, coherent English text with correct primary facts (e.g., Paris as capital), though SVD-LLM(W) shows factual errors (Eiffel Tower height, tallest structure claim) and a spelling mistake. At 60\% parameter rate, performance remains comparable with both methods producing interpretable output, though SVD-LLM(W) shows factual inaccuracies (department count) and awkward phrasing. At 40\% parameter rate, both methods show degradation: SVD-LLM(W) exhibits terminology errors (``Romantic'' instead of Romance) and factual underestimations, while \AIR{} produces multiple factual errors (most widely spoken language, UN status) alongside typographical issues (``Romanoic''). Both methods still generate coherent English sentences at this compression level.

At 20\% parameter rate, the methods diverge. SVD-LLM(W) produces empty responses with no tokens generated. \AIR{} continues to generate grammatically structured English sentences. \AIR{}'s content quality degrades with fabricated information (non-existent International Assistance Program), neologisms (``differentities''), and nonsensical continuations, but the model retains its language generation capability and syntactic structure. \AIR{} thus avoids catastrophic failure under high compression, though the quality of its outputs at this level is substantially degraded. At 20\% parameters, \AIR{} still correctly identifies ``Paris'' as the capital of France, while SVD-LLM(W) produces no output.

{\scriptsize
\setlength{\tabcolsep}{3pt}
\begin{longtable}{@{}cL{4.0cm}L{4.0cm}L{3.5cm}@{}}
\toprule
\multicolumn{4}{l}{\textbf{Combined comparison: SVD-LLM(W)+LoRA vs.\ \AIR{}+LoRA}} \\
\midrule
\multicolumn{4}{l}{\scriptsize\textbf{Color legend:}\;
\sethlcolor{errorred}\hl{Factual error}\;
\sethlcolor{grammarorange}\hl{Grammar/spelling issues}\;
\sethlcolor{collapsepurple}\hl{Complete collapse}} \\
\midrule
\textsc{Param. Rate} & \textsc{SVD-LLM(W)+LoRA} & \textsc{\AIR{}+LoRA} & \textsc{Comment} \\
\midrule
\endhead
\multicolumn{4}{l}{\textit{\textbf{Input-Sample: ``The capital of France is''}}} \\
\midrule
80\% & 
Paris. The Eiffel Tower is a famous landmark in Paris, and it was built by \hlerror{Gustave Eifel} in 1889. It stands at 324 meters tall and is \hlerror{the tallest structure in Europe}. & 
Paris. The capital city of the United States is Washington, D.C. The capitol building in Washington is called the Capitol Building. & 
Both correct on Paris. SVD-LLM: spelling error (Eifel$\rightarrow$Eiffel); Eiffel Tower not tallest structure in Europe. \AIR{}: topic shift to U.S. capitals; content itself is accurate. \\
\midrule
60\% & 
Paris. The French government is a democratic republic, with the President as head of state and the Prime Minister as \hlgrammar{the head minister}. France has \hlerror{56 departments}. & 
Paris. The French Revolution was a political upheaval that occurred in 1789 and lasted for several years... & 
Both correct on Paris. SVD-LLM: awkward phrasing (``head minister'' instead of head of government), wrong department count (actually 101 including overseas). \AIR{}: topic shift to French Revolution; content is historically accurate. \\
\midrule
40\% & 
Paris. The French language is a \hlgrammar{Romantic} language, and it has been used in many countries... with an estimated \hlerror{10 million speakers} in the United... & 
Paris. The French language is \hlerror{the most widely spoken language in Europe}, with an estimated \hlerror{120 million speakers} worldwide. It is also \hlerror{the official language of... the UN}. French is a \hlgrammar{Romanoic} language. & 
Both identify Paris. SVD-LLM: ``Romantic'' should be Romance; ``10 million speakers'' is implausible (too high if intended as US speakers, $\sim$1--2M; far too low if global, $\sim$280M). \AIR{}: ``Romanoic'' typo; false claim (English/Russian more widely spoken in Europe); 120M too low; French is one of 6 UN languages but not ``the'' official language. \\
\midrule
20\% & 
\hlcollapse{(empty, no tokens generated)} & 
Paris, which is the city of Paris. \hlerror{In 2015, the International Assistance Program (International Assistive System) was launched by the United States and Canada to assist the rescue of injured soldiers in the Middle East.} & 
SVD-LLM produces nothing. \AIR{} identifies Paris correctly but with tautological phrasing and completely fabricated content about a non-existent program. \\
\midrule
\multicolumn{4}{l}{\textit{\textbf{Input-Sample: ``The universe is''}}} \\
\midrule
80\% & 
a vast and complex system, and understanding it requires the development of sophisticated mathematical models. The goal of this course is to introduce students to the basic concepts of mathematics that are used in modeling physical phenomena. We will learn about... & 
a vast and complex system, and the study of it requires a great deal of knowledge. The more we know about our surroundings, the better we can understand them and make use of them. This is why science is so important. It... & 
Both coherent. SVD-LLM: shifts to academic course description. \AIR{}: smooth philosophical continuation. \\
\midrule
60\% & 
a complex system of interconnected processes, and the study of these processes is known as physics. Physics is often divided into several subfields: particle physics, astrophysics, cosmology, \hlgrammar{electromagnetic physics (electronics)}... & 
a complex system of interconnected processes, and the study of these processes is known as physics. The field of physics encompasses many subfields that focus on different aspects of physical phenomena. Some of the most well-known subdis... & 
Very similar responses. SVD-LLM: ``electromagnetic physics (electronics)'' is non-standard terminology (should be electromagnetism; electronics is engineering). \AIR{}: coherent, standard phrasing. \\
\midrule
40\% & 
so vast and complex. But there are some things that we know about, like gravity, light, and time. These basic laws of physics play a central role in our understanding of the universe. The fundamental laws: 1. Newton... & 
a vast and complex system of interconnected stars, planets, galaxies, and other objects. The Sun is the most visible star in our solar system, but it is only one among billions of stars that exist in the galaxy. The... & 
Both coherent and on-topic. SVD-LLM: minor grammatical issue (``so vast'' lacks explicit subject); otherwise reasonable. \AIR{}: factually accurate and well-structured. \\
\midrule
20\% & 
\hlcollapse{(empty, no tokens generated)} & 
a complex object, and it can be viewed as an array of different objects. Here are some examples of the various ways to see the \hlgrammar{differentities} of objects: \hlgrammar{a, a, b, c, d, or d}. These are all... & 
SVD-LLM: empty output. \AIR{}: neologism ``differentities'' (made-up word); nonsensical enumeration; but maintains grammatical structure and English-like output. \\
\bottomrule
\end{longtable}
}

\subsection{Spatial patterns in weight and relevance distributions}\label{sec:appendix_spatial_patterns}

\begin{figure*}[t!]
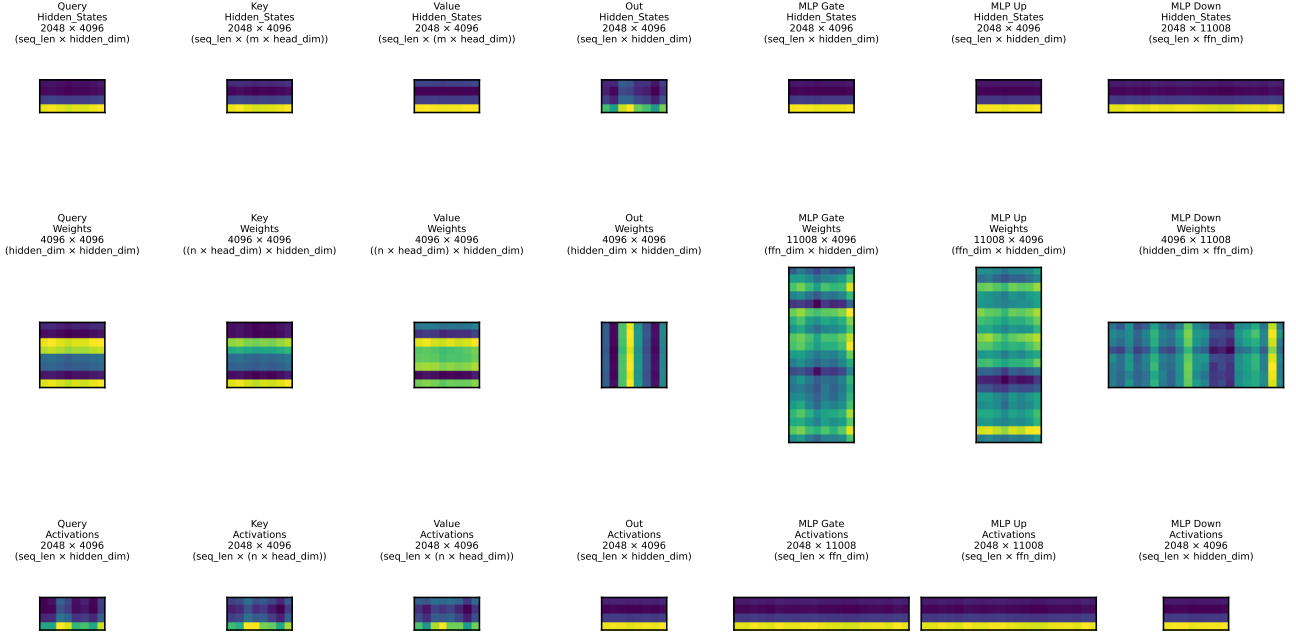

    \centering
    \vspace{1ex}
    \begin{flushleft}
        \small{\underline{Native Space}}
    \end{flushleft}
    \includegraphics[
        width=1.0\textwidth,
        page=8,
        trim={{0.19\paperwidth} {0.1\paperheight} {0.13\paperwidth} {0.07\paperheight}},
        clip
    ]{graphics/intra_layer_distributions/plot_Pseudo_Relevances_weight_space_grid_combined.pdf}
    
    \vspace{1ex}
    \begin{flushleft}
        \small{\underline{Relevance Space}}
    \end{flushleft}
    \includegraphics[
        width=1.0\textwidth,
        page=8,
        trim={{0.19\paperwidth} {0.1\paperheight} {0.13\paperwidth} {0.07\paperheight}},
        clip
    ]{graphics/intra_layer_distributions/plot_Relevances_weight_space_grid_combined.pdf}    
    \caption{
        Spatial heatmaps in layer block 10 of LLaMA 7B shown for weight magnitudes (upper figure) and relevance magnitudes (lower figure). Each sub-figure shows hidden states, weights, and activations (rows) across attention layers (Q,K,V,O) and MLP layers (Gate,Up,Down) (columns). In a simplified manner, activations A emerge based on hidden-states X and weights W through $A = X \cdot W^\top$. Due to the large matrix dimensions (up to $11008\times 4096$), we aggregate values over $512\times 512$ blocks, with yellow indicating high accumulations and purple indicating low accumulations.
    }
    \label{fig:relevance_distribution_inter_layer}
\end{figure*}

Figure~\ref{fig:relevance_distribution_inter_layer} shows spatial patterns in the distribution of functional influence for the LRP-$\epsilon$ default instantiation of $\mathbf{I}$ (the figure is computed from this specific signal). While weight magnitudes exhibit relatively uniform distributions across matrix dimensions, influence magnitudes show concentrated regions of high influence with structural patterns. This spatial concentration is most evident in MLP layers, where influence forms distinct bands and clusters rather than being uniformly distributed. These patterns indicate that functional influence is not randomly distributed but follows systematic spatial organization, providing a visual explanation for why \AIR{}'s influence-weighted penalty redistributes approximation error to less important regions of the weight matrix.

\clearpage

\end{document}